\documentclass{article}

\usepackage{microtype}
\usepackage{graphicx}
\usepackage{subfigure}
\usepackage{booktabs} 

\usepackage{hyperref}

\usepackage[accepted]{icml2025}

\usepackage{amsmath}
\usepackage{amssymb}
\usepackage{mathtools}
\usepackage{amsthm}
\usepackage{thmtools, thm-restate}

\theoremstyle{plain}
\newtheorem{theorem}{Theorem}[section]
\newtheorem{proposition}[theorem]{Proposition}
\newtheorem{lemma}[theorem]{Lemma}
\newtheorem{corollary}[theorem]{Corollary}
\theoremstyle{definition}

\theoremstyle{remark}

\usepackage{amsfonts}
\usepackage{xcolor}

\newcommand{\mc}{\mathcal}
\newcommand{\mb}{\mathbb}

\usepackage{amsmath}
\DeclareMathOperator*{\argmax}{arg\,max}

\usepackage{appendix}
\usepackage{titletoc}

\usepackage{etoc}
\usepackage{pgfplots}
\usetikzlibrary{pgfplots.groupplots} 
\pgfplotsset{compat=1.17}

\usepackage[disable,textsize=tiny]{todonotes}

\icmltitlerunning{A Two-Stage Learning-to-Defer Approach for Multi-Task Learning}

\begin{document}

\twocolumn[
\icmltitle{A Two-Stage Learning-to-Defer Approach for Multi-Task Learning}



\icmlsetsymbol{equal}{*}

\begin{icmlauthorlist}
\icmlauthor{Yannis Montreuil}{yannis,ng,create,equal}
\icmlauthor{Shu Heng Yeo}{yannis,equal}
\icmlauthor{Axel Carlier}{axel,ipal}
\icmlauthor{Lai Xing Ng}{ng,ipal}
\icmlauthor{Wei Tsang Ooi}{yannis,ipal}
\end{icmlauthorlist}

\icmlaffiliation{yannis}{School of Computing, National University of Singapore, Singapore}
\icmlaffiliation{axel}{IRIT, Université de Toulouse, CNRS, Toulouse INP, Toulouse, France}
\icmlaffiliation{ng}{Institute for Infocomm Research, Agency for Science, Technology and Research, Singapore}
\icmlaffiliation{ipal}{IPAL, IRL2955, Singapore}
\icmlaffiliation{create}{CNRS@CREATE LTD, 1 Create Way, Singapore}

\icmlcorrespondingauthor{Yannis Montreuil}{yannis.montreuil@u.nus.edu}

\hypersetup{
  pdftitle={A Two-Stage Learning-to-Defer Approach for Multi-Task Learning},
  pdfauthor={Yannis Montreuil, Shu Heng Yeo, Axel Carlier, Lai Xing Ng, Wei Tsang Ooi}
}

\icmlkeywords{Learning-to-Defer, learning with abstention, learning theory}

\vskip 0.3in
]



\printAffiliationsAndNotice{\icmlEqualContribution} 



\begin{abstract}
    The Two-Stage Learning-to-Defer (L2D) framework has been extensively studied for classification and, more recently, regression tasks. However, many real-world applications require solving both tasks jointly in a multi-task setting. We introduce a novel Two-Stage L2D framework for multi-task learning that integrates classification and regression through a unified deferral mechanism. Our method leverages a two-stage surrogate loss family, which we prove to be both Bayes-consistent and $(\mathcal{G}, \mathcal{R})$-consistent, ensuring convergence to the Bayes-optimal rejector. We derive explicit consistency bounds tied to the cross-entropy surrogate and the $L_1$-norm of agent-specific costs, and extend minimizability gap analysis to the multi-expert two-stage regime. We also make explicit how shared representation learning---commonly used in multi-task models—affects these consistency guarantees. Experiments on object detection and electronic health record analysis demonstrate the effectiveness of our approach and highlight the limitations of existing L2D methods in multi-task scenarios.
\end{abstract}

\section{Introduction}

Learning-to-Defer (L2D) integrates predictive models with human experts—or, more broadly, decision-makers—to optimize systems requiring high reliability~\citep{madras2018predict}. This approach benefits from the scalability of machine learning models and leverages expert knowledge to address complex queries~\citep{hemmer2021human}. The Learning-to-Defer approach defers decisions to experts when the learning-based model has lower confidence than the most confident expert. This deference mechanism enhances safety, which is particularly crucial in high-stakes scenarios \citep{mozannar2021consistent, Mozannar2023WhoSP}. For example, in medical diagnostics, the system utilizes patient-acquired data to deliver an initial diagnosis \citep{johnson2023mimic, johnson2016mimic}. If the model is sufficiently confident, its diagnosis is accepted; otherwise, the decision is deferred to a medical expert who provides the final diagnosis. Such tasks, which can directly impact human lives, underscore the need to develop reliable systems \citep{balagurunathan2021requirements}.

Learning-to-Defer has been extensively studied in classification problems \citep{madras2018predict, Verma2022LearningTD, mozannar2021consistent, Mozannar2023WhoSP, mao2023twostage} and, more recently, in regression scenarios \citep{mao2024regressionmultiexpertdeferral}. However, many modern complex tasks involve both regression and classification components, requiring deferral to be applied to both components simultaneously, as they cannot be treated independently. For instance, in object detection, a model predicts both the class of an object and its location using a regressor, with these outputs being inherently interdependent \citep{girshick2015fast, redmon2016lookonceunifiedrealtime, buch2017sst}. In practice, deferring only localization or classification is not meaningful, as decision-makers will treat these two tasks simultaneously. A failure in either component—such as misclassifying the object or inaccurately estimating its position—can undermine the entire problem, emphasizing the importance of coordinated deferral strategies that address both components jointly.

This potential for failure underscores the need for a Learning-to-Defer approach tailored to multi-task problems involving both classification and regression. We propose a novel framework for multi-task environments, incorporating expertise from multiple experts and the predictor-regressor model. We focus our work on the \textit{two-stage scenario}, where the model is already trained offline. This setting is relevant when retraining from scratch the predictor-regressor model is either too costly or not feasible due to diverse constraints such as non-open models \citep{mao2023twostage, mao2024regressionmultiexpertdeferral}. We approximate the \textit{true deferral loss} using a \textit{surrogate deferral loss} family, based on cross-entropy, and tailored for the two-stage setting, ensuring that the loss effectively approximates the original discontinuous loss function. Our theoretical analysis establishes that our surrogate loss is both $(\mathcal{G}, \mathcal{R})$-consistent and Bayes-consistent. Furthermore, we study and generalize results on the minimizability gap for deferral loss based on cross-entropy, providing deeper insights into its optimization properties. Our contributions are as follows:

(i) \textbf{Novelty}: We introduce two-stage Learning-to-Defer for multi-task learning with multiple experts. Unlike previous L2D methods that focus solely on classification or regression, our approach addresses situations where a sole optimal agent has to be selected to jointly handle both tasks in a unified framework. 


(ii) \textbf{Theoretical Foundation}: We prove that our surrogate family is both Bayes-consistent and \( (\mathcal{G}, \mathcal{R}) \)-consistent for any cross-entropy-based surrogate. We derive tight consistency bounds that depend on the choice of the surrogate and the \( L_1 \)-norm of the cost, extending minimizability gap analysis to the two-stage, multi-expert setting. Additionally, we establish learning bounds for the \textit{true deferral loss}, showing that generalization improves as agents become more accurate.

(iii) \textbf{Empirical Validation}: We evaluate our approach on two challenging tasks. In object detection, our method effectively captures the intrinsic interdependence between classification and regression, overcoming the limitations of existing L2D approaches. In EHR analysis, we show that current L2D methods struggle when agents have varying expertise across classification and regression—whereas our method achieves superior performance.

\section{Related Work}

Learning-to-Defer builds upon the foundational ideas of \emph{Learning with Abstention}~\citep{Chow_1970, Bartlett_Wegkamp_2008, cortes_rejection, Geifman_El-Yaniv_2017, Ramaswamy, Cao, Mao_Mohri_Zhong_2023_abst}, where a model is permitted to abstain from making a prediction when its confidence is low. The core insight of L2D is to extend this framework from rejection to deferral—delegating uncertain decisions to external agents or experts whose confidence may exceed that of the model.

\paragraph{One-stage Learning-to-Defer.}
L2D was originally introduced by~\citet{madras2018predict} for binary classification, using a \emph{pass function} inspired by the \emph{predictor-rejector} framework of~\citet{cortes_rejection}. In the multiclass setting, \citet{mozannar2021consistent} proposed a \emph{score-based} formulation that leverages a log-softmax surrogate to ensure Bayes-consistency. This formulation has since been extended to a wide range of classification tasks~\citep{okati2021differentiable, Verma2022LearningTD, Cao_Mozannar_Feng_Wei_An_2023, Cao, Keswani, Kerrigan, Hemmer, Benz, Tailor, liu2024mitigating, palomba2024causal, wei2024exploiting}. A pivotal contribution by~\citet{Mozannar2023WhoSP} challenged the sufficiency of Bayes-consistency, showing that existing score-based methods may be suboptimal under realizable distributions—particularly when the hypothesis class is restricted. They introduced the notion of \emph{hypothesis-consistency}, which strengthens theoretical alignment between the surrogate loss and the constrained hypothesis space. This work sparked a broader effort to refine the theoretical foundations of L2D using tools from surrogate risk analysis~\citep{pmlr-v28-long13, Zhang, Awasthi_Mao_Mohri_Zhong_2022_multi, mao2023crossentropylossfunctionstheoretical}. Recent theoretical advances have solidified the status of score-based L2D. \citet{mao2024principled} established that the general score-based L2D framework achieves \(\mathcal{H}\)-consistency, while \citet{mao2024realizablehconsistentbayesconsistentloss, mao2025mastering} introduced a novel surrogate loss that guarantees \emph{realizable-consistency}—i.e., optimality under realizable distributions. \citet{montreuil2025onestagetopklearningtodeferscorebased} generalize L2D to deferral to the set of top-$k$ experts. Beyond classification, the L2D framework has also been extended to regression~\citep{mao2024regressionmultiexpertdeferral}, demonstrating its applicability in continuous-output settings with expert deferral.

\paragraph{Two-stage Learning-to-Defer.}
The emergence of large-scale pretrained models has motivated the development of \emph{two-stage} L2D frameworks, where both the model and the expert agents are trained offline. This reflects practical constraints: most users lack the computational resources to fine-tune large models end-to-end.  \citet{narasimhan2022post} were the first to formalize this setting, and \citet{mao2023twostage} introduced a dedicated \emph{predictor–rejector} architecture tailored for two-stage L2D, with theoretical guarantees including both Bayes- and hypothesis-consistency. \citet{charusaie2022sample} offered a comparative analysis of one-stage (joint training) and two-stage (post hoc) L2D, highlighting trade-offs between model flexibility and sample efficiency. More recently, two-stage L2D has been successfully extended to regression~\citep{mao2024regressionmultiexpertdeferral} and top-$k$ expert deferral \citep{montreuil2025askaskktwostage}, and has been applied to real-world tasks such as extractive question answering~\citep{montreuil2024learningtodeferextractivequestionanswering} and adversarial robustness~\citep{montreuil2025adversarialrobustnesstwostagelearningtodefer}.

Despite significant progress, current two-stage L2D research largely addresses classification and regression independently. However, many contemporary tasks involve both regression and classification components, necessitating their joint optimization. In this work, we extend two-stage L2D to joint classifier-regressor models, addressing this critical gap.

\section{Preliminaries}

\paragraph{Multi-task scenario.} \label{prel:multi}
We consider a multi-task setting encompassing both classification and regression problems. Let $\mathcal{X}$ denote the input space, $\mathcal{Y} = \{1, \ldots, n\}$ represent the set of $n$ distinct classes, and $\mathcal{T} \subseteq \mathbb{R}$ denote the space of real-valued targets for regression. For compactness, each data point is represented as a triplet $z = (x, y, t) \in \mathcal{Z}$, where $\mathcal{Z} = \mathcal{X} \times \mathcal{Y} \times \mathcal{T}$. We assume the data is sampled independently and identically distributed (i.i.d.) from a distribution $\mathcal{D}$ over $\mathcal{Z}$ \citep{girshick2015fast,redmon2016lookonceunifiedrealtime, detr}. 

We define a \emph{backbone} $w \in \mathcal{W}$, or shared feature extractor, such that $w: \mathcal{X} \to \mathcal{Q}$. For example, $w$ can be a deep network that takes an input $x \in \mathcal{X}$ and produces a latent feature vector $q = w(x) \in \mathcal{Q}$. Next, we define a \emph{classifier} $h \in \mathcal{H}$, representing all possible classification heads operating on $\mathcal{Q}$. Formally, $h$ is a score function defined as $h: \mathcal{Q} \times\mc{Y} \to \mb{R}$, where the predicted class is $h(q)=\arg\max_{y \in \mathcal{Y}} h(q, y)$. Likewise, we define a \emph{regressor} $f \in \mathcal{F}$, representing all regression heads, where $f: \mathcal{Q} \to \mathcal{T}$. These components are combined into a single multi-head network $g \in \mathcal{G}$, where $\mathcal{G} = \{\, g : g(x) = (h \circ w(x), f \circ w(x)) \mid w \in \mathcal{W},\, h \in \mathcal{H},\, f \in \mathcal{F} \}$. Hence, $g$ jointly produces classification and regression outputs, $h(q)$ and $f(q)$, from the same latent representation $q = w(x)$.

\paragraph{Consistency in classification.} 
In the classification setting, the goal is to identify a classifier \( h \in \mathcal{H} \) in the specific case where \( w(x) = x \), such that \( h(x) = \arg\max_{y \in \mathcal{Y}} h(x, y) \). This classifier should minimize the true error \( \mathcal{E}_{\ell_{01}}(h) \), defined as \( \mathcal{E}_{\ell_{01}}(h) = \mathbb{E}_{(x, y)}\big[\ell_{01}(h(x), y)\big] \).  The Bayes-optimal error is given by $\mathcal{E}^B_{\ell_{01}}(\mathcal{H}) = \inf_{h \in \mathcal{H}} \mathcal{E}_{\ell_{01}}(h)$. However, directly minimizing $\mathcal{E}_{\ell_{01}}(h)$ is challenging due to the non-differentiability of the \emph{true multiclass} 0-1 loss \citep{Statistical, Steinwart2007HowTC, Awasthi_Mao_Mohri_Zhong_2022_multi, cortes2025balancing, mao2025principled}. This motivates the introduction of the cross-entropy \emph{multiclass surrogate} family, denoted by $\Phi_{01}^\nu: \mathcal{H} \times \mathcal{X} \times \mathcal{Y} \to \mathbb{R}^+$, which provides a convex upper bound to the \textit{true multiclass loss} $\ell_{01}$. This family is parameterized by $\nu \geq 0$ and encompasses standard surrogate functions widely adopted in the community such as the MAE \citep{Ghosh2017RobustLF} or the log-softmax \citep{Mohri}.
\begin{equation}\label{eq:comp-sum}
    \begin{aligned}
        \Phi_{01}^\nu = & \begin{cases}
         \frac{1}{1-\nu}\Big( \Big[ \sum_{y'\in\mc{Y}}e^{h(x,y') - h(x,y)}]^{1-\nu} -1\Big) & \nu \not =1 \\
         \log\Big( \sum_{y'\in\mc{Y}}e^{h(x,y') - h(x,y)} \Big) & \nu =1.
    \end{cases}
    \end{aligned}
\end{equation}
The corresponding surrogate error is defined as $\mathcal{E}_{\Phi_{01}^\nu}(h) = \mathbb{E}_{(x, y)} \big[\Phi_{01}^\nu(h(x), y)\big]$, with its optimal value given by $\mathcal{E}^\ast_{\Phi_{01}^\nu}(\mathcal{H}) = \inf_{h \in \mathcal{H}} \mathcal{E}_{\Phi_{01}^\nu}(h)$.  A crucial property of a surrogate loss is \emph{Bayes-consistency}, which guarantees that minimizing the surrogate generalization error also minimizes the true generalization error \citep{Statistical, Steinwart2007HowTC, bartlett1, tewari07a, mao2025enhanced, zhong2025fundamental}. Formally, $\Phi_{01}^\nu$ is Bayes-consistent with respect to $\ell_{01}$ if, for any sequence $\{h_k\}_{k \in \mathbb{N}} \subset \mathcal{H}$, the following implication holds:
\begin{equation}
\begin{aligned}\label{bayes-consi}
    & \mathcal{E}_{\Phi_{01}^\nu}(h_k) - \mathcal{E}_{\Phi_{01}^\nu}^\ast(\mathcal{H}) \xrightarrow{k \to \infty} 0 \\
    & \implies \mathcal{E}_{\ell_{01}}(h_k) - \mathcal{E}_{\ell_{01}}^B(\mathcal{H}) \xrightarrow{k \to \infty} 0.
\end{aligned}
\end{equation}
This property assumes that $\mathcal{H} = \mathcal{H}_{\text{all}}$, a condition that does not necessarily hold for restricted hypothesis classes such as $\mathcal{H}_{\text{lin}}$ or $\mathcal{H}_{\text{ReLU}}$ \citep{pmlr-v28-long13, Awasthi_Mao_Mohri_Zhong_2022_multi}. To address this limitation, \citet{Awasthi_Mao_Mohri_Zhong_2022_multi} proposed $\mathcal{H}$-consistency bounds. These bounds depend on a non-decreasing function $\Gamma: \mathbb{R}^+ \to \mathbb{R}^+$ and are expressed as:
\begin{equation}\label{mhbc}
\begin{aligned}
     & \mathcal{E}_{\Phi_{01}^\nu}(h) - \mathcal{E}_{\Phi_{01}^\nu}^\ast(\mathcal{H}) + \mathcal{U}_{\Phi_{01}^\nu}(\mathcal{H}) \geq \\
     & \Gamma\Big(\mathcal{E}_{\ell_{01}}(h) - \mathcal{E}_{\ell_{01}}^B(\mathcal{H}) + \mathcal{U}_{\ell_{01}}(\mathcal{H})\Big),
\end{aligned}
\end{equation}
where the minimizability gap $\mathcal{U}_{\ell_{01}}(\mathcal{H})$ measures the disparity between the best-in-class generalization error and the expected pointwise minimum error: $\mathcal{U}_{\ell_{01}}(\mathcal{H}) = \mathcal{E}_{\ell_{01}}^B(\mathcal{H}) - \mathbb{E}_{x} \big[ \inf_{h \in \mathcal{H}} \mathbb{E}_{y \mid x} [\ell_{01}(h(x), y)] \big]$. Notably, the minimizability gap vanishes when $\mathcal{H} = \mathcal{H}_{\text{all}}$ \citep{Steinwart2007HowTC, Awasthi_Mao_Mohri_Zhong_2022_multi, cortes2024cardinality, mao2024multi, mao2024h, mao2024universal}. In the asymptotic limit, inequality \eqref{mhbc} guarantees the recovery of Bayes-consistency, aligning with the condition in \eqref{bayes-consi}.

\section{Two-stage Multi-Task L2D: Theoretical Analysis}
\subsection{Formulating the True Deferral Loss}

We extend the two-stage predictor–rejector framework, originally proposed by \cite{narasimhan2022post, mao2023twostage}, to the multi-task setting described in Section~\ref{prel:multi}. Specifically, we consider an \textit{offline-trained model} \( g \in \mc{G} \), which jointly performs classification and regression. In addition, we assume access to \( J \) \textit{offline-trained experts}, denoted \( \text{M}_j \) for \( j \in \{1, \dots, J\} \). Each expert outputs predictions of the form \( m_j(x) = \big(m_j^h(x), m_j^f(x)\big) \), where \( m_j^h(x) \in \mc{Y} \) and \( m_j^f(x) \in \mc{T} \) correspond to the classification and regression components, respectively. Each expert prediction lies in a corresponding space \( \mathcal{M}_j \), so that \( m_j(x) \in \mathcal{M}_j \). We denote the aggregated outputs of all experts as \( m(x) = \big(m_1(x), \dots, m_J(x)\big) \in \mathcal{M} := \prod_{j=1}^J \mathcal{M}_j \). We write \( [J] := \{1, \dots, J\} \) to denote the index set of experts, and define the set of all agents as \( \mc{A} := \{0\} \cup [J] \), where agent 0 corresponds to the model \( g \). Thus, the system contains \( |\mc{A}| = J + 1 \) agents in total.

To allocate each decision, we introduce a \emph{rejector} function \( r \in \mathcal{R} \), where \( r : \mathcal{X} \times \mathcal{A} \to \mathbb{R} \). Given an input \( x \in \mathcal{X} \), the rejector selects the agent \( j \in \mc{A} \) that maximizes its score: \( r(x) = \argmax_{j \in \mc{A}} r(x, j) \). This mechanism induces the \emph{deferral loss}, a mapping \( \ell_{\text{def}} : \mathcal{R} \times \mathcal{G} \times \mathcal{Z} \times \mathcal{M} \to \mathbb{R}_+ \), which quantifies the cost of allocating a decision to a particular agent.

\begin{restatable}[True deferral loss]{definition}{l2d}\label{def_l2d} Let an input $x\in\mc{X}$, for any $r\in\mc{R}$, we have the \textit{true deferral loss}:
\begin{equation*}
    \begin{aligned} \label{base:deferral}
    \ell_{\text{def}}(r,g,m,z) & = \sum_{j=0}^Jc_j(g(x),m_j(x),z)1_{r(x)=j},
    \end{aligned}
\end{equation*}
\end{restatable}
with a bounded cost \( c_j \) that quantifies the penalty incurred when allocating the decision to agent \( j \in \mc{A} \). When the rejector \( r \in \mc{R} \) predicts \( r(x) = 0 \), the decision is assigned to the multi-task model \( g \), incurring a base cost \( c_0 \) defined as \( c_0(g(x), z) = \rho(g(x), z) \), where \( \rho(\cdot, \cdot) \in \mb{R}_+ \) measures the discrepancy between the model's output \( g(x) \) and the ground truth \( z \). Conversely, if the rejector selects \( r(x) = j \) for some \( j > 0 \), the decision is deferred to expert \( j \), yielding a deferral cost \( c_j(m_j(x), z) = \rho(m_j(x), z) + \beta_j \). Here, \( \beta_j \geq 0 \) denotes the querying cost associated with invoking expert \( j \), which may reflect domain-specific constraints such as computational overhead, annotation effort, or time expenditure.

When the classification and regression objectives are separable, the total cost can be decomposed as \( c_j = \lambda^{\text{cla}} c^{\text{cla}} + \lambda^{\text{reg}} c^{\text{reg}} \), where \( \lambda^{\text{cla}}, \lambda^{\text{reg}} \geq 0 \) specify the relative importance of each task. A neutral setting is recovered when \( \lambda^{\text{cla}} = \lambda^{\text{reg}} = 1 \), ensuring a task-agnostic trade-off. If classification performance is prioritized, one can select \( \lambda^{\text{cla}} > \lambda^{\text{reg}} \) to favor agents with stronger classification expertise.

\paragraph{Optimal deferral rule.} In Definition~\ref{def_l2d}, we introduced the \emph{true deferral loss} \( \ell_{\text{def}} \), which quantifies the expected cost incurred when allocating predictions across the model and experts. Our goal is to minimize this loss by identifying the Bayes-optimal rejector \( r \in \mc{R} \) that minimizes the true risk. To formalize this objective, we analyze the \emph{pointwise Bayes rejector} \( r^B(x) \), which minimizes the conditional risk \( \mc{C}_{\ell_{\text{def}}} \). The corresponding population risk is given by \( \mathcal{E}_{\ell_{\text{def}}}(g, r) = \mathbb{E}_x[\mathcal{C}_{\ell_{\text{def}}}(g, r, x)] \). The following lemma characterizes the optimal decision rule at each input \( x \in \mc{X} \).
\begin{restatable}[Pointwise Bayes Rejector]{lemma}{fta}
\label{lemma_pointwise}
Given an input \( x \in \mathcal{X} \) and data distribution \( \mathcal{D} \), the rejection rule that minimizes the conditional risk \( \mathcal{C}_{\ell_{\text{def}}} \) associated with the true deferral loss \( \ell_{\text{def}} \) is:
\begin{equation*} \label{eq:bayes_rejector}
r^B(x) =
\begin{cases}
    0 & \text{if } \displaystyle \inf_{g\in\mc{G}}\mb{E}_{y,t|x}[c_0] \leq \min_{j \in [J]} \mathbb{E}_{y,t \mid x}  \left[ c_j \right]  \\
    j & \text{otherwise},
\end{cases}
\end{equation*}
\end{restatable}
The proof is provided in Appendix~\ref{proof_rejector}. Lemma~\ref{lemma_pointwise} shows that the optimal rejector \( r \in \mathcal{R} \) assigns the decision to the model \( g \in \mc{G} \) whenever its expected cost is lower than that of any expert. Otherwise, the rejector defers to the expert with the minimal expected deferral cost.

Although Lemma~\ref{lemma_pointwise} characterizes the Bayes-optimal policy under the true deferral loss \( \ell_{\text{def}} \), this loss is non-differentiable and thus intractable for direct optimization in practice \citep{Statistical}.

\subsection{Surrogate Loss for Two-Stage Multi-Task L2D}

\paragraph{Introducing the surrogate.} To address the optimization challenges posed by discontinuous losses \citep{Berkson1944ApplicationOT, Cortes1995SupportVector}, we introduce a family of convex surrogate losses with favorable analytical properties. Specifically, we adopt the multiclass cross-entropy surrogates \( \Phi_{01}^\nu : \mathcal{R} \times \mathcal{X} \times \mathcal{A} \to \mathbb{R}_+ \), which upper-bounds the true multiclass 0-1 loss \( \ell_{01} \) and facilitates gradient-based optimization. This surrogate family is defined in Equation~\ref{eq:comp-sum}.

Building on the framework of \citet{mao2024regressionmultiexpertdeferral}, who proposed convex surrogates for deferral settings, we extend their approach to account for the interdependence between classification and regression tasks. In our setting, this yields a family of surrogate losses \( \Phi^\nu_{\text{def}} : \mathcal{R} \times \mathcal{G} \times \mathcal{M} \times \mathcal{Z} \to \mathbb{R}_+ \), which incorporate the full structure of the multi-task cost.

\begin{restatable}[Surrogate Deferral Surrogates]{lemma}{lemmasurrogate}
\label{surr:defer}
Let \( x \in \mathcal{X} \) and let \( \Phi_{01}^\nu \) be a multiclass surrogate loss. Then the surrogate deferral loss \( \Phi^\nu_{\text{def}} \) for \( J+1 \) agents is given by
\begin{equation*}
\Phi^\nu_{\text{def}}(r, g, m, z) = \sum_{j=0}^{J} \tau_j(g(x), m(x), z) \, \Phi_{01}^\nu(r, x, j),
\end{equation*}
where the aggregated cost weights are defined as $\tau_j(g(x), m(x), z) = \sum_{i=0}^{J} c_i(g(x), m_i(x), z) 1_{i \neq j}$.
\end{restatable}

The surrogate deferral loss \( \Phi^\nu_{\text{def}} \) combines the individual surrogate losses \( \Phi_{01}^\nu(r, x, j) \) for each agent \( j \in \mc{A} \), weighted by the corresponding aggregated cost \( \tau_j \). Intuitively, \( \tau_0 \) quantifies the total cost of deferring to any expert instead of using the model, while \( \tau_j \) for \( j > 0 \) reflects the total cost incurred by selecting expert \( j \) instead of any other agent, including the model and other experts.

This construction preserves task generality and only requires that the base surrogate \( \Phi_{01}^\nu \) admit an \( \mathcal{R} \)-consistency bound. The modular formulation of the cost functions \( c_j \) allows this surrogate to flexibly accommodate diverse multi-task settings.

\paragraph{Consistency of the surrogate losses.} In Lemma~\ref{surr:defer}, we established that the proposed surrogate losses form a convex upper bound on the \emph{true deferral loss} \( \ell_{\text{def}} \). However, it remains to determine whether this surrogate family provides a reliable approximation of the true loss in terms of optimal decision-making. In particular, it is not immediate that the pointwise minimizer of the surrogate loss, \( r^*(x) \), aligns with the Bayes-optimal rejector \( r^B(x) \) that minimizes \( \ell_{\text{def}} \). To address this, we study the relationship between the surrogate and true risks by analyzing their respective \emph{excess risks}. Specifically, we compare the surrogate excess risk, \( \mathcal{E}_{\Phi^\nu_{\text{def}}}(g, r) - \mathcal{E}^*_{\Phi^\nu_{\text{def}}}(\mathcal{G}, \mathcal{R}) \), to the true excess risk, \( \mathcal{E}_{\ell_{\text{def}}}(g, r) - \mathcal{E}^B_{\ell_{\text{def}}}(\mathcal{G}, \mathcal{R}) \). Understanding this discrepancy is crucial for establishing the \((\mc{G}, \mc{R})\)-consistency of the surrogate loss family, a topic extensively studied in prior work on multiclass surrogate theory \citep{Steinwart2007HowTC, Statistical, bartlett1, Awasthi_Mao_Mohri_Zhong_2022_multi}.

Leveraging consistency bounds developed in \citep{Awasthi_Mao_Mohri_Zhong_2022_multi, mao2024enhanced}, we present Theorem~\ref{theo:consistency}, which proves that the surrogate deferral loss family \( \Phi^\nu_{\text{def}} \) is indeed \((\mc{G}, \mc{R})\)-consistent.

\begin{restatable}[$(\mathcal{G}, \mathcal{R})$-consistency bounds]{theorem}{consistency}
\label{theo:consistency}
Let \( g \in \mathcal{G} \) be a multi-task model. Suppose there exists a non-decreasing function \( \Gamma^\nu : \mathbb{R}_+ \to \mathbb{R}_+ \), parameterized by \( \nu \geq 0 \), such that the \( \mathcal{R} \)-consistency bound holds for any distribution \( \mathcal{D} \):
\begin{equation*}
\begin{aligned}
    & \mathcal{E}_{\Phi_{01}^\nu}(r) - \mathcal{E}^*_{\Phi_{01}^\nu}(\mathcal{R}) + \mathcal{U}_{\Phi_{01}^\nu}(\mathcal{R})
    \geq \\
    &\Gamma^\nu\left( \mathcal{E}_{\ell_{01}}(r) - \mathcal{E}^B_{\ell_{01}}(\mathcal{R}) + \mathcal{U}_{\ell_{01}}(\mathcal{R}) \right),
\end{aligned}
\end{equation*}
then for any \( (g, r) \in \mathcal{G} \times \mathcal{R} \), any distribution \( \mathcal{D} \), and any \( x \in \mathcal{X} \),
\begin{equation*}\label{eq:cons}
\begin{aligned}
    &\mathcal{E}_{\ell_{\text{def}}}(g, r ) - \mathcal{E}^B_{\ell_{\text{def}}}( \mathcal{G}, \mathcal{R}) + \mathcal{U}_{\ell_{\text{def}}}( \mathcal{G}, \mathcal{R}) \leq \\ 
    &\quad \overline{\Gamma}^\nu\left(\mathcal{E}_{\Phi^\nu_{\text{def}}}(r) - \mathcal{E}^*_{\Phi^\nu_{\text{def}}}(\mathcal{R}) + \mathcal{U}_{\Phi^\nu_{\text{def}}}(\mathcal{R})\right) \\
    & \quad + \mathcal{E}_{c_0}(g) - \mathcal{E}^B_{c_0}(\mathcal{G}) + \mathcal{U}_{c_0}(\mathcal{G}),
\end{aligned}
\end{equation*}
where the expected aggregated cost vector is given by \( \boldsymbol{\overline{\tau}} = \big( \mathbb{E}_{y,t \mid x}[\tau_0], \dots, \mathbb{E}_{y,t \mid x}[\tau_J] \big) \), and $\overline{\Gamma}^\nu(u) = \|\boldsymbol{\overline{\tau}}\|_1 \Gamma^\nu\left( \frac{u}{\|\boldsymbol{\overline{\tau}}\|_1} \right)$
with \( \Gamma^\nu(u) = \mathcal{T}^{-1,\nu}(u) \). In the case of the log-softmax surrogate (\( \nu = 1 \)), the transformation is given by \( \mathcal{T}^{\nu=1}(u) = \frac{1+u}{2} \log(1+u) + \frac{1-u}{2} \log(1-u) \).
\end{restatable}

The proof of Theorem~\ref{theo:consistency}, along with generalizations to any \( \nu \geq 0 \), is provided in Appendix~\ref{proof_consistency}. This result yields refined consistency guarantees for the surrogate deferral loss, improving upon the bounds established by \citet{mao2024regressionmultiexpertdeferral}. The bounds are explicitly tailored to the cross-entropy surrogate family and parameterized by \( \nu \), allowing for precise control over the surrogate's approximation behavior. Crucially, the tightness of the bound depends on the aggregated deferral costs, and is scaled by the \( L_1 \)-norm \( \| \boldsymbol{\overline{\tau}} \|_1 \), which quantifies the cumulative cost discrepancy across agents.

Moreover, we show that the surrogate deferral losses are \( (\mathcal{G}, \mathcal{R}) \)-consistent whenever the underlying multiclass surrogate family \( \Phi_{01}^\nu \) is \( \mathcal{R} \)-consistent. Under the assumption that \( \mathcal{R} = \mathcal{R}_{\text{all}} \) and \( \mathcal{G} = \mathcal{G}_{\text{all}} \), the minimizability gaps vanish, as established by \citet{Steinwart2007HowTC}. As a result, minimizing the \emph{surrogate deferral excess risk} while accounting for the minimizability gap yields $\mathcal{E}_{\Phi^\nu_{\text{def}}}(r_k) - \mathcal{E}^*_{\Phi^\nu_{\text{def}}}(\mathcal{R}_{\text{all}}) + \mathcal{U}_{\Phi^\nu_{\text{def}}}(\mathcal{R}_{\text{all}}) \xrightarrow{k \to \infty} 0$. Since the multi-task model \( g \) is trained offline, it is reasonable to assume that the \( c_0 \)-excess risk also vanishes: $\mathcal{E}_{c_0}(g_k) - \mathcal{E}^B_{c_0}(\mathcal{G}_{\text{all}}) + \mathcal{U}_{c_0}(\mathcal{G}_{\text{all}}) \xrightarrow{k \to \infty} 0$. Combining the two convergence results and invoking the properties of \( \overline{\Gamma}^\nu \), we conclude that
\[
\mathcal{E}_{\ell_{\text{def}}}(g, r_k) - \mathcal{E}^B_{\ell_{\text{def}}}(\mathcal{G}_{\text{all}}, \mathcal{R}_{\text{all}}) + \mathcal{U}_{\ell_{\text{def}}}(\mathcal{G}_{\text{all}}, \mathcal{R}_{\text{all}}) \xrightarrow{k \to \infty} 0.
\]
Hence, the following corollary holds:

\begin{restatable}[Bayes-consistency of the deferral surrogate losses]{corollary}{bayesconsistency}
\label{corr:bayesconsistency}
Under the conditions of Theorem~\ref{theo:consistency}, and assuming \( (\mathcal{G}, \mathcal{R}) = (\mathcal{G}_{\text{all}}, \mathcal{R}_{\text{all}}) \) and \( \mathcal{E}_{c_0}(g_k) - \mathcal{E}^B_{c_0}(\mathcal{G}_{\text{all}}) \xrightarrow{k \to \infty} 0 \), the \emph{surrogate deferral loss} family \( \Phi^\nu_{\text{def}} \) is Bayes-consistent with respect to the true deferral loss \( \ell_{\text{def}} \). Specifically, minimizing the surrogate deferral excess risk ensures convergence of the true deferral excess risk. Formally, for sequences \( \{r_k\}_{k \in \mathbb{N}} \subset \mathcal{R} \) and \( \{g_k\}_{k \in \mathbb{N}} \subset \mathcal{G} \), we have:
\begin{equation*}
\begin{aligned}
    & \mathcal{E}_{\Phi^\nu_{\text{def}}}(r_k) - \mathcal{E}_{\Phi^\nu_{\text{def}}}^*(\mathcal{R}_{\text{all}}) \xrightarrow{k \to \infty} 0 \\
    & \quad \Longrightarrow \quad \mathcal{E}_{\ell_{\text{def}}}(g_k, r_k) - \mathcal{E}_{\ell_{\text{def}}}^B(\mathcal{G}_{\text{all}}, \mathcal{R}_{\text{all}}) \xrightarrow{k \to \infty} 0.
\end{aligned}
\end{equation*}
\end{restatable}

This result confirms that, as \( k \to \infty \), the surrogate losses \( \Phi^\nu_{\text{def}} \) attain asymptotic Bayes optimality for both the rejector \( r \) and the offline-trained multi-task model \( g \). Thus, the surrogate family faithfully approximates the true deferral loss in the limit. Moreover, the pointwise surrogate-optimal rejector \( r^*(x) \) converges to a close approximation of the Bayes-optimal rejector \( r^B(x) \), thereby inducing deferral decisions consistent with the characterization in Lemma~\ref{lemma_pointwise} \citep{bartlett1}.

\paragraph{Analysis of the minimizability gap.} 

In As shown by \citet{Awasthi_Mao_Mohri_Zhong_2022_multi}, the minimizability gap does not vanish in general. Understanding the conditions under which it arises, quantifying its magnitude, and identifying effective mitigation strategies are crucial for ensuring that surrogate-based optimization aligns with the true task-specific objectives.

We provide a novel and strong characterization of the minimizability gap in the two-stage setting with multiple experts, extending the results of \citet{mao2024theoretically}, who analyzed the gap in the context of learning with abstention (constant cost) for a single expert and a specific distribution.

\begin{restatable}[Characterization of Minimizability Gaps]{theorem}{minimizability} 
\label{minimizability} 
Assume \( \mathcal{R} \) is symmetric and complete. Then, for the cross-entropy multiclass surrogates \( \Phi_{01}^\nu \) and any distribution \( \mathcal{D} \), the following holds for \( \nu \geq 0 \):
\begin{equation*}
\begin{aligned}
    \mathcal{C}_{\Phi^\nu_{\text{def}}}^{\nu, \ast}(\mc{R},x) = 
    \begin{cases}
        \|\overline{\boldsymbol{\tau}}\|_1 H\left(\frac{\overline{\boldsymbol{\tau}}}{\|\overline{\boldsymbol{\tau}}\|_1}\right) & \mspace{-10mu} \text{for } \nu = 1 \\
        \|\overline{\boldsymbol{\tau}}\|_1 - \|\overline{\boldsymbol{\tau}}\|_\infty & \mspace{-10mu}  \nu = 2 \\
        \frac{1}{\nu - 1} \left[ \|\overline{\boldsymbol{\tau}}\|_1 - \|\overline{\boldsymbol{\tau}}\|_{\frac{1}{2-\nu}} \right] & \mspace{-10mu} \nu \in (1,2) \\
        \frac{1}{1-\nu} \Big[ \Big(\sum_{k=0}^J \overline{\tau}_k^{\frac{1}{2-\nu}}\Big)^{2-\nu} \mspace{-30mu} - \|\overline{\boldsymbol{\tau}}\|_1 \Big] & \mspace{-10mu}  \nu > 2,
    \end{cases}
\end{aligned}
\end{equation*}
where \( \overline{\boldsymbol{\tau}} = \{\mathbb{E}_{y,t|x}[\overline{\tau}_0], \dots, \mathbb{E}_{y,t|x}[\overline{\tau}_J]\} \), the aggregated costs are \( \tau_j = \sum_{k=0}^J c_k 1_{k \neq j} \), and \( H \) denotes the Shannon entropy.  The minimizability gap is defined as $\mathcal{U}_{\Phi^\nu_{\text{def}}}(\mathcal{R}) = \mathcal{E}^*_{\Phi^\nu_{\text{def}}}(\mathcal{R}) - \mathbb{E}_x\left[ \mathcal{C}_{\Phi^\nu_{\text{def}}}^{\nu, \ast}(\mc{R},x)\right]$.
\end{restatable}
We provide the proof in Appendix~\ref{proof_minimi}. Theorem~\ref{minimizability} characterizes the minimizability gap \( \mathcal{U}_{\Phi^\nu_{\text{def}}}(\mathcal{R}) \) for cross-entropy multiclass surrogates over symmetric and complete hypothesis sets \( \mathcal{R} \). The gap depends on \( \nu \geq 0 \), and its behavior varies across different surrogates. Specifically, for \( \nu = 1 \), the gap is proportional to the Shannon entropy of the normalized expected cost vector \( \frac{\overline{\boldsymbol{\tau}}}{\|\overline{\boldsymbol{\tau}}\|_1} \), which increases with entropy, reflecting higher uncertainty in the misclassification distribution. For \( \nu = 2 \), the gap simplifies to the difference between the \( L_1 \)-norm and \( L_\infty \)-norm of \( \overline{\boldsymbol{\tau}} \), where a smaller gap indicates concentrated misclassifications, thus reducing uncertainty. For \( \nu \in (1,2) \), the gap balances the entropy-based sensitivity at \( \nu = 1 \) and the margin-based sensitivity at \( \nu = 2 \). As \( \nu \to 1^+ \), the gap emphasizes agents with higher misclassification counts, while as \( \nu \to 2^- \), it shifts towards aggregate misclassification counts. For \( \nu < 1 \), where \( p = \frac{1}{2 - \nu} \in (0,1) \), the gap becomes more sensitive to misclassification distribution, increasing when errors are dispersed. For \( \nu > 2 \), with \( p < 0 \), reciprocal weighting reduces sensitivity to dominant errors, potentially decreasing the gap but at the risk of underemphasizing critical misclassifications.

In the special case of learning with abstention and a single expert (\( J = 1 \)), assigning costs \( \tau_0 = 1 \) and \( \tau_J = 1 - c \) recovers the minimizability gap introduced by \citet{mao2024theoretically}. Thus, our formulation generalizes the minimizability gap to settings with multiple experts, non-constant costs, and arbitrary distributions \( \mathcal{D} \).

\subsection{Encoder--Aware Bounds}
\label{subsec:extractor-aware}
In this section, we show that our approach is theoretically aligned with multi-task learning using shared representations. Let \( \mathcal{W} \) denote a class of representation functions (encoders), \( \mathcal{H} \) a class of classification heads, and \( \mathcal{F} \) a class of regression heads. For any \( (w, h, f) \in \mathcal{W} \times \mathcal{H} \times \mathcal{F} \), the multi-task predictor is defined as \( g_{w,h,f}(x) = \big(h \circ w(x),\; f \circ w(x)\big) \), where the shared representation \( w(x) \) is passed to both task-specific heads.  The true risk defined as \( \mathcal{E}_{c_0}(g) = \mathbb{E}_{z \sim \mathcal{D}} \left[ c_0(g(x), (y, t)) \right] \), where \( z = (x, y, t) \in \mathcal{X} \times \mathcal{Y} \times \mathcal{T} \). The Bayes risk over a class \( \mathcal{G} \) is given by \( \mathcal{E}_{c_0}^{B}(\mathcal{G}) = \inf_{g \in \mathcal{G}} \mathcal{E}_{c_0}(g) \).

\begin{proposition}[Head and representation gaps.] Fix $w\in\mathcal W$ and let $\mathcal G:=\{g_{w',h',f'}:w'\in\mathcal W,\,h'\in\mathcal H,\,f'\in\mathcal F\}$.  Define
\begin{align*}
  \mathcal E_{\min}(w) &:= \inf_{h',f'}\mathcal E_{c_0}\bigl(g_{w,h',f'}\bigr), \\
  \Delta_{\mathrm{heads}}(w,h,f) &:= \mathcal E_{c_0}\bigl(g_{w,h,f}\bigr)-\mathcal E_{\min}(w),\\
  \Delta_{\mathrm{repr}}(w) &:= \mathcal E_{\min}(w)-\mathcal E_{c_0}^{B}(\mathcal G).
\end{align*}
\end{proposition}
The quantity $\Delta_{\mathrm{heads}}$ measures \emph{head sub--optimality} given the extracted representation from the encoder fixed at a particular iteration, while $\Delta_{\mathrm{repr}}$ captures how far $w$ lies from a Bayes--optimal shared representation.

\begin{lemma}[Non--negativity of the gaps]\label{lem:gaps-nonnegativity}
For all $(w,h,f)\in\mathcal W\times\mathcal H\times\mathcal F$, we have $\Delta_{\mathrm{heads}}(w,h,f)\ge 0$.  For every $w\in\mathcal W$, we have $\Delta_{\mathrm{repr}}(w)\ge 0$.
\end{lemma}
\begin{proof}
Fix $w$.  By definition of the infimum, $\mathcal E_{\min}(w)=\inf_{h',f'}\mathcal E_{c_0}\bigl(g_{w,h',f'}\bigr)\le \mathcal E_{c_0}\bigl(g_{w,h,f}\bigr)$ for any heads $(h,f)$, hence $\Delta_{\mathrm{heads}}(w,h,f)\ge 0$.  For the representation gap, note that $\mathcal E_{c_0}^{B}(\mathcal G)=\inf_{w',h',f'}\mathcal E_{c_0}\bigl(g_{w',h',f'}\bigr)\le \mathcal E_{\min}(w)$, so $\Delta_{\mathrm{repr}}(w)=\mathcal E_{\min}(w)-\mathcal E_{c_0}^{B}(\mathcal G)\ge 0$.  Both inequalities hold with equality when $(w,h,f)$ is Bayes--optimal.
\end{proof}

\begin{proposition}[Excess--risk decomposition]\label{prop:extractor-decomp}
For every $(w,h,f)$,
\[
  \mathcal E_{c_0}\bigl(g_{w,h,f}\bigr)-\mathcal E_{c_0}^{B}(\mathcal G)
  \;=\;\Delta_{\mathrm{heads}}(w,h,f)+\Delta_{\mathrm{repr}}(w).
\]
\end{proposition}
\begin{proof}
Add and subtract $\mathcal E_{\min}(w)$.
\end{proof}

\begin{proposition}[Cost of Enforcing a Shared Encoder] Suppose two \emph{independent} heads act directly on the raw input $x$:
\(
  g_{\text{sep},h,f}(x)=(h(x),f(x)).
\)
Let $\mathcal E_{\text{sep},c_0}^{B}:=\inf_{h,f}\mathcal E_{c_0}\bigl(g_{\text{sep},h,f}\bigr)$ and define
\[
  \Delta_{\mathrm{MTL}}:=\mathcal E_{c_0}^{B}(\mathcal G)-\mathcal E_{\text{sep},c_0}^{B}.
\]
\end{proposition}

Hence $\Delta_{\mathrm{MTL}}<0$ indicates that forcing a \emph{shared} encoding is beneficial, whereas $\Delta_{\mathrm{MTL}}>0$ points to a potential penalty relative to two stand--alone models.

Combining definitions,
\[
  \mathcal E_{c_0}(g_{w,h,f})-\mathcal E_{c_0}^{B}(\mathcal G)
  \;=\;\bigl[\mathcal E_{c_0}(g_{w,h,f})-\mathcal E_{\text{sep},c_0}^{B}\bigr]-\Delta_{\mathrm{MTL}},
\]  
we can link these relationships with the main Theorem \ref{theo:consistency}  that states that for any $(g,r)\in\mathcal G\times\mathcal R$.  Setting $g=g_{w,h,f}$ and invoking Proposition~\ref{prop:extractor-decomp} yields the \emph{encoder--aware consistency bound}:

\begin{corollary}[Encoder--aware $(\mathcal G,\mathcal R)$--consistency]\label{cor:extractor-consistency}
For any $(w,h,f,r)\in\mathcal W\times\mathcal H\times\mathcal F\times\mathcal R$,
\begin{align*}
  &\mathcal E_{\ell_{\mathrm{def}}}(g_{w,h,f},r)-\mathcal E_{\ell_{\mathrm{def}}}^{B}(\mathcal G,\mathcal R)+\mc{U}_{\ell_{\mathrm{def}}}(\mathcal G,\mathcal R)
  \le \\
  & \quad \overline{\Gamma}^{\nu}\bigl(\mathcal E_{\Phi^{\nu}_{\mathrm{def}}}(r)-\mathcal E_{\Phi^{\nu}_{\mathrm{def}}}^{B}(\mathcal R)+\mc{U}_{\Phi^{\nu}_{\mathrm{def}}}(\mathcal R)\bigr)\\
         & \quad +\,\Delta_{\mathrm{heads}}(w,h,f)+\Delta_{\mathrm{repr}}(w)+\mc{U}_{c_0}(\mathcal G).\label{eq:extractor_consistency}
\end{align*}
\end{corollary}
Corollary~\eqref{cor:extractor-consistency} decomposes the end--to--end excess deferral risk into three \emph{orthogonal} sources: (i) the rejector optimisation error, (ii) the head sub--optimality, and (iii) the representation gap.

The bound suggests a two--stage pipeline: (i) learn or select high--capacity representations to minimise $\Delta_{\mathrm{repr}}$ as well as best heads for this representation, then (ii) optimise the rejector to tighten the remaining terms. The pipeline exactly mirrors our proposed L2D solution. Such decoupling is particularly attractive when $|\mathcal T|$ is large and feature sharing is essential for sample efficiency.

\subsection{Generalization Bound}
We aim to quantify the generalization capability of our system, considering both the complexity of the hypothesis space and the quality of the participating agents. To this end, we define the empirical optimal rejector \( \widehat{r}^B \) as the minimizer of the empirical generalization error:
\begin{equation}
    \widehat{r}^B = \arg\min_{r \in \mathcal{R}} \frac{1}{K} \sum_{k=1}^K \ell_{\text{def}}(g, m, r, z_k),
\end{equation}
where \( \ell_{\text{def}} \) denotes the \textit{true deferral loss} function. To characterize the system's generalization ability, we utilize the Rademacher complexity, which measures the expressive richness of a hypothesis class by evaluating its capacity to fit random noise \citep{Bartlett_rademecher, Mohri}. The proof of Lemma \ref{lemma_learning} is provided in Appendix \ref{proof_lemma}.

\begin{restatable}{lemma}{lemmalearning}\label{lemma_learning} Let \( \mathcal{L}_1 \) be a family of functions mapping \( \mathcal{X} \) to \( [0,1] \), and let \( \mathcal{L}_2 \) be a family of functions mapping \( \mathcal{X} \) to \( \{0,1\} \). Define \( \mathcal{L} = \{ l_1 l_2 : l_1 \in \mathcal{L}_1, l_2 \in \mathcal{L}_2 \} \). Then, the empirical Rademacher complexity of \( \mathcal{L} \) for any sample \( S \) of size \( K \) is bounded by:
\begin{equation*}
    \widehat{\mathfrak{R}}_S(\mc{L}) \leq \widehat{\mathfrak{R}}_S(\mc{L}_1) + \widehat{\mathfrak{R}}_S(\mc{L}_2).
\end{equation*}    
\end{restatable}
For simplicity, we assume costs $c_0(g(x),z) = \ell_{01}(h(x),y) + \ell_{\text{reg}}(f(x),t)$ and $c_{j>0}(m_j(x),z) = c_0(m_j(x),z)$. We assume the regression loss $\ell_{\text{reg}}$ to be non-negative, bounded by $L$, and Lipschitz. Furthermore, we assume that \(m^h_{k,j}\) is drawn from the conditional distribution of the random variable \(M^h_j\) givenparameters \(\{X = x_k, Y = y_k\}\), and that \(m^f_{k,j}\) is drawn from the conditional distribution of \(M^f_j\) given \(\{X = x_k, T = t_k\}\). We define the family of deferral loss functions 
as \(\mathcal{L}_{\text{def}} 
= \{\ell_{\text{def}} : \mathcal{G}\times \mathcal{R}\times \mathcal{M}\times \mathcal{Z} \to \mb{R}^+\}\).  Under these assumptions, we derive the generalization bounds for the binary setting as follows:

\begin{restatable}[Learning bounds of the deferral loss]{theorem}{learningbounds} \label{learning_bounds} For any expert M$_j$, any distribution $\mc{D}$ over $\mc{Z}$, we have with probability $1-\delta$ for $\delta\in[0,1/2]$, that the following bound holds at the optimum:
\begin{equation*}
    \begin{aligned}
        \mc{E}_{\ell_{\text{def}}}(h,f,r) \leq \widehat{\mc{E}}_{\ell_{\text{def}}}(h,f,r) + 2\mathfrak{R}_K(\mc{L}_{\text{def}}) + \sqrt{\frac{\log 1/\delta}{2K}},
    \end{aligned}
\end{equation*}
with \begin{equation*}
    \begin{aligned}
        \mathfrak{R}_K(\mc{L}_{\text{def}}) & \leq \frac{1}{2}\mathfrak{R}_K(\mc{H}) + \mathfrak{R}_K(\mc{F}) + \sum_{j=1}^J \Omega(m_j^h, y)\\
        & + \Big(\sum_{j=1}^J\max\ell_{\text{reg}}(m_j^f, t) + 2\Big) \mathfrak{R}_K(\mc{R}),
    \end{aligned}
\end{equation*}
with $\Omega(m_j^h, y)=\frac{1}{2}\mc{D}(m_j^h \neq y) \exp\left(-\frac{K}{8} \mc{D}(m_j^h \neq y)\right) + \mathcal{R}_{K \mc{D}(m_j^h \neq y)/2}(\mathcal{R})$.
\end{restatable}
We prove Theorem \ref{learning_bounds} in Appendix \ref{proof_learning}. The terms \( \mathfrak{R}_K(\mathcal{H}) \) and \( \mathfrak{R}_K(\mathcal{F}) \) denote the Rademacher complexities of the hypothesis class \( \mathcal{H} \) and function class \( \mathcal{F} \), respectively, indicating that the generalization bounds depend on the complexity of the pre-trained model.  The term \( \Omega(m_j^h, y) \) captures the impact of each expert's classification error on the learning bound. It includes an exponentially decaying factor, \( \frac{\mathcal{D}(m_j^h \neq y)}{2} \exp \left( -\frac{K \mathcal{D}(m_j^h \neq y)}{8} \right) \), which decreases rapidly as the sample size \( K \) grows or as the expert's error rate \( \mathcal{D}(m_j^h \neq y) \) declines \citep{mozannar2021consistent}. This reflects the intuition that more accurate experts contribute less to the bound, improving overall generalization. Finally, the last term suggests that the generalization properties of our \textit{true deferral loss} depend on the expert’s regression performance.

\section{Experiments}
In this section, we present the performance improvements achieved by the proposed Learning-to-Defer surrogate in a multi-task context. Specifically, we demonstrate that our approach excels in object detection, a task where classification and regression components are inherently intertwined and cannot be delegated to separate agents, and where existing L2D methods encounter significant limitations. Furthermore, we evaluate our approach on an Electronic Health Record task, jointly predicting mortality (classification) and length of stay (regression), comparing our results with \citet{mao2023twostage, mao2024regressionmultiexpertdeferral}.

For each experiment, we report the mean and standard deviation across four independent trials to account for variability in the results. All training and evaluation were conducted on an NVIDIA H100 GPU. We give our training algorithm in Appendix \ref{appendix:algo}. Additional figures and details are provided in Appendix~\ref{appendix:exp}. To ensure reproducibility, we have made our implementation publicly available.

\subsection{Object Detection Task} \label{object}
We evaluate our approach using the Pascal VOC dataset \citep{pascal}, a multi-object detection benchmark. This is the first time such a multi-task problem has been explored within the L2D framework as previous L2D approaches require the classification and regression component to be independent \citep{mao2023twostage, mao2024regressionmultiexpertdeferral}.

\paragraph{Dataset and Metrics:} The PASCAL Visual Object Classes (VOC) dataset \citep{pascal} serves as a widely recognized benchmark in computer vision for evaluating object detection models. It consists of annotated images spanning 20 object categories, showcasing diverse scenes with varying scales, occlusions, and lighting conditions.  To assess object detection performance, we report the mean Average Precision (mAP), a standard metric in the field. Additionally, in the context of L2D, we report the allocation metric (All.), which represents the ratio of allocated queries per agent.

\paragraph{Agents setting:} We trained three distinct Faster R-CNN models \citep{ren2016fasterrcnnrealtimeobject} to serve as our agents, differentiated by their computational complexities. The smallest, characterized by $\text{GFLOPS} = 12.2$, represents our model $g\in\mc{G}$ with $\mathcal{G} = \{\, g : g(x) = (h \circ w(x), f \circ w(x)) \mid w \in \mathcal{W},\, h \in \mathcal{H},\, f \in \mathcal{F} \}$. The medium-sized, denoted as Expert 1, has a computational cost of $\text{GFLOPS} = 134.4$, while the largest, Expert 2, operates at $\text{GFLOPS} = 280.3$. To account for the difference in complexity between Experts 1 and 2, we define the ratio $R_G = 280.3 / 134.4$ and set the query cost for Expert 1 as $\beta_1 = \beta_2 / R_G$. This parameterization reflects the relative computational costs of querying experts.  We define the agent costs as \( c_0(g(x), z) = \text{mAP}(g(x), z) \) and \( c_{j \in [J]}(m_j(x), z) = \text{mAP}(m_j(x), z) \). We report the performance metrics of the agents alongside additional training details in Appendix~\ref{exp:pascal}.

\paragraph{Rejector:} The rejector is trained using a smaller version of the Faster R-CNN model \citep{ren2016fasterrcnnrealtimeobject}. Training is performed for $200$ epochs using the Adam optimizer~\citep{kingma2017adammethodstochasticoptimization} with a learning rate of $0.001$ and a batch size of $64$. The checkpoint achieving the lowest empirical risk on the validation set is selected for evaluation.

\begin{figure}
    \centering
    \begin{tikzpicture}
\begin{groupplot}[
    group style={
        group size=1 by 2,
        vertical sep=0pt    
    },
    width=6cm,
    height=3.5cm,
    xmin=0, xmax=0.52,
]

\nextgroupplot[
    ylabel={\tiny mAP},
    ymin=0.38, ymax=0.55,
    xtick=\empty,
    x tick label style={draw=none},
    legend style={at={(0.98,0.98)},anchor=north east, font=\tiny},
    minor y tick num=4
]
\addplot[
    blue,
    mark=o,
    mark size=1pt,
    error bars/.cd,
    y dir=both,
    y explicit
]
table [x index=0, y index=1, y error index=2, col sep=space] {
    0.01 0.528 0.000
    0.05 0.525 0.000
    0.10 0.491 0.006
    0.15 0.442 0.004
    0.20 0.420 0.002
    0.30 0.401 0.002
    0.50 0.395 0.000
};
\addlegendentry{L2D};

\addplot[
    red,
    mark=none,
    dashed,
    line width=1pt,
    error bars/.cd,
    y dir=both,
    y explicit
]
table [x index=0, y index=1, y error index=2, col sep=space] {
    0.00 0.395 0.000
    0.05 0.395 0.000
    0.10 0.395 0.000
    0.15 0.395 0.000
    0.20 0.395 0.000
    0.30 0.395 0.000
    0.50 0.395 0.000
};
\addlegendentry{Model};

\addplot[
    teal,
    mark=none,
    dashed,
    line width=1pt,
    error bars/.cd,
    y dir=both,
    y explicit
]
table [x index=0, y index=1, y error index=2, col sep=space] {
    0.00 0.433 0.000
    0.05 0.433 0.000
    0.10 0.433 0.000
    0.15 0.433 0.000
    0.20 0.433 0.000
    0.30 0.433 0.000
    0.50 0.433 0.000
};
\addlegendentry{Expert 1};

\addplot[
    magenta,
    mark=none,
    dashed,
    line width=1pt,
    error bars/.cd,
    y dir=both,
    y explicit
]
table [x index=0, y index=1, y error index=2, col sep=space] {
    0.00 0.528 0.000
    0.05 0.528 0.000
    0.10 0.528 0.000
    0.15 0.528 0.000
    0.20 0.528 0.000
    0.30 0.528 0.000
    0.50 0.528 0.000
};
\addlegendentry{Expert 2};
\nextgroupplot[
    xlabel={\tiny Cost $\beta_2$},
    ylabel={\tiny Allocation Ratio},
    ymin=-0.05, ymax=1.05,
    legend style={at={(0.98,0.5)},anchor=east, font=\tiny},
    minor y tick num=4,
    minor x tick num=1
]
\addplot[
    red,
    mark=o,
    mark size=1pt,
    error bars/.cd,
    y dir=both,
    y explicit
]
table [x index=0, y index=1, y error index=2, col sep=space] {
    0.01 0.00 0.000
    0.05 0.073 0.008
    0.10 0.480 0.007
    0.15 0.681 0.003
    0.20 0.775 0.002
    0.30 0.981 0.000
    0.50 1.00 0.000
};
\addlegendentry{Model};

\addplot[
    teal,
    mark=triangle*,
    mark size=1.5pt,
    error bars/.cd,
    y dir=both,
    y explicit
]
table [x index=0, y index=1, y error index=2, col sep=space] {
    0.01 0.00 0.000
    0.05 0.00 0.00
    0.10 0.00 0.00
    0.15 0.197 0.004
    0.20 0.225 0.005
    0.30 0.019 0.001
    0.50 0.00 0.000
};
\addlegendentry{Expert1};

\addplot[
    magenta,
    mark=diamond*,
    mark size=1.5pt,
    error bars/.cd,
    y dir=both,
    y explicit
]
table [x index=0, y index=1, y error index=2, col sep=space] {
    0.01 1.00 0.000
    0.05 0.927 0.003
    0.10 0.520 0.002
    0.15 0.122 0.001
    0.20 0.00 0.000
    0.30 0.00 0.000
    0.50 0.00 0.000
};
\addlegendentry{Expert2};
\end{groupplot}
\end{tikzpicture}
    \caption{Performance comparison across different cost values $\beta_2$ on Pascal VOC \citep{pascal}. The table reports the mean Average Precision (mAP) and the allocation ratio for the model and two experts with mean and variance. We report these results in Appendix Table \ref{table:results_allocation}.}
    \label{fig:multitask}
\end{figure}


\paragraph{Results:} In Figure \ref{fig:multitask}, we observe that for lower cost values, specifically when $\beta_1 < 0.15$, the system consistently avoids selecting Expert 1. This outcome arises because the cost difference between $\beta_1$ and $\beta_2$ is negligible, making it more advantageous to defer to Expert 2 (the most accurate expert), where the modest cost increase is offset by superior outcomes. When $\beta_2 = 0.15$, however, it becomes optimal to defer to both experts and model at the same time. In particular, there exist instances $x \in \mathcal{X}$ where both Expert 1 and Expert 2 correctly predict the target (while the model does not). In such cases, Expert 1 is preferred due to its lower cost $\beta_1 < \beta_2$. Conversely, for instances $x \in \mathcal{X}$ where Expert 2 is accurate and Expert 1 (along with the model) is incorrect, the system continues to select Expert 2, as $\beta_2$ remains relatively low. For $\beta_2 \geq 0.2$, the increasing cost differential between the experts shifts the balance in favor of Expert 1, enabling the system to achieve strong performance while minimizing overall costs.

This demonstrates that our approach effectively allocates queries among agents, thereby enhancing the overall performance of the system, even when the classification and regression tasks are interdependent.

\subsection{EHR Task}\label{exp:ehr}

We compare our novel approach against existing two-stage L2D methods \citep{mao2023twostage, mao2024regressionmultiexpertdeferral}. Unlike the first experiment on object detection (Subsection \ref{object}), where classification and regression tasks are interdependent, this evaluation focuses on a second scenario where the two tasks can be treated independently.

\paragraph{Dataset and Metrics:} The Medical Information Mart for Intensive Care IV (MIMIC-IV) dataset \citep{johnson2023mimic} is a comprehensive collection of de-identified health-related data  patients admitted to critical care units. For our analysis, we focus on two tasks: mortality prediction and length-of-stay prediction, corresponding to classification and regression tasks, respectively. To evaluate performance, we report accuracy (Acc) for the mortality prediction task, which quantifies classification performance, and Smooth L1 loss (sL1) for the length-of-stay prediction task, which measures the deviation between the predicted and actual values. Additionally, we report the allocation metric (All.) for L2D to capture query allocation behavior.

\paragraph{Agents setting:} We consider two experts,  M\(_1 \) and  M\(_2 \), acting as specialized agents, aligning with the category allocation described in \citep{mozannar2021consistent, Verma2022LearningTD, Verma_Nalisnick_2022, Cao_Mozannar_Feng_Wei_An_2023}. The dataset is partitioned into \( Z = 6 \) clusters using the \( K \)-means algorithm \citep{kmeans}, where \( Z \) is selected via the Elbow method \citep{Thorndike1953WhoBI}. The clusters are denoted as \( \{C_1, C_2, \dots, C_Z\} \). Each cluster represents a subset of data instances grouped by feature similarity, enabling features-specific specialization by the experts. The experts are assumed to specialize in distinct subsets of clusters based on the task. For classification,  M\(_1 \) correctly predicts the outcomes for clusters \( C^{\text{M}_1}_{\text{cla}} = \{C_1, C_2, C_4\} \), while  M\(_2 \) handles clusters \( C^{\text{M}_2}_{\text{cla}} = \{C_1, C_5, C_6\} \). Notably, cluster \( C_1 \) is shared between the two experts, reflecting practical scenarios where domain knowledge overlaps. For regression tasks,  M\(_1 \) is accurate on clusters \( C^{\text{M}_1}_{\text{reg}} = \{C_1, C_3, C_5\} \), while  M\(_2 \) specializes in clusters \( C^{\text{M}_2}_{\text{reg}} = \{C_1, C_4, C_6\} \). Here too, overlap is modeled, with cluster \( C_1 \) being common to both experts and classification-regression task. Note that the category assignments do not follow any specific rule.  

We assume that each expert produces correct predictions for the clusters they are assigned \citep{Verma2022LearningTD, mozannar2021consistent}. Conversely, for clusters outside their expertise, predictions are assumed to be incorrect. In such cases, for length-of-stay predictions, the outcomes are modeled using a uniform probability distribution to reflect uncertainty. The detailed performance evaluation of these agents is provided in Appendix \ref{appendix:ehr}.
 
The model utilizes two compact transformer architectures \citep{trans} for addressing both classification and regression tasks, formally defined as $\mathcal{G} = \{\, g : g(x) = (h(x), f(x)) \mid h \in \mathcal{H},\, f \in \mathcal{F} \,\}$. The agent's costs are specified as $c_0(g(x), z) = \lambda^{\text{cla}} \ell_{01}(h(x),y) + \lambda^{\text{reg}} \ell_{\text{reg}}(f(x),t)$ and $c_{j \in [J]}(m_j(x), z) = c_0(m_j(x), z) + \beta_j$. Consistent with prior works \citep{mozannar2021consistent, Verma2022LearningTD, mao2023twostage, mao2024regressionmultiexpertdeferral}, we set $\beta_j = 0$.

\paragraph{Rejectors:} The two-stage L2D rejectors are trained using a small transformer model \citep{trans} as the encoder, following the approach outlined by \citet{pyhealth2023yang}, with a classification head for query allocation. Training is performed over 100 epochs with a learning rate of 0.003, a warm-up period of 0.1, a cosine learning rate scheduler, the Adam optimizer \citep{kingma2017adammethodstochasticoptimization}, and a batch size of 1024 for all baselines. The checkpoint with the lowest empirical risk on the validation set is selected for evaluation.

\paragraph{Results:} Table \ref{table:mimic} compares the performance of our proposed Learning-to-Defer (Ours) approach with two existing methods: a classification-focused rejector \citep{mao2023twostage} and a regression-focused rejector \citep{mao2024regressionmultiexpertdeferral}. The results highlight the limitations of task-specific rejectors and the advantages of our balanced approach.

\begin{table}[htbp]
\centering\resizebox{0.48\textwidth}{!}{ 
\begin{tabular}{@{}lccccc@{}}
\toprule
Rejector & Acc (\%) & sL1 & All. Model & All. Expert 1 & All. Expert 2 \\
\midrule
\citet{mao2023twostage} & $71.3\pm.1$ & $1.45\pm.03$ & $.60\pm.02$ & $.01\pm.01$ & $.39\pm.02$ \\
\citet{mao2024regressionmultiexpertdeferral} & $50.7\pm.8$ & $1.18\pm.05$ & $.38\pm.01$ & $.37\pm.02$ & $.25\pm.01$ \\
Ours & $70.0\pm.5$ & $1.28\pm.02$ & $.66\pm.01$ & $.12\pm.02$ & $.22\pm.01$ \\
\bottomrule
\end{tabular}}
\caption{Performance comparison of different two-stage L2D. The table reports accuracy (Acc), smooth L1 loss (sL1), and allocation rates (All.) to the model and experts with mean and variance.}
\label{table:mimic}
\end{table}

The classification-focused rejector achieves the highest classification accuracy at $71.3\%$ but struggles with regression, as reflected by its high smooth L1 loss of $ 1.45$. On the other hand, the regression-focused rejector achieves the best regression performance with an sL1 loss of 1.18 but performs poorly in classification with an accuracy of $50.7\%$. In contrast, our method balances performance across tasks, achieving a classification accuracy of 70.0\% and an sL1 loss of 1.28. Moreover, it significantly reduces reliance on experts, allocating 66\% of queries to the model compared to 60\% for \citet{mao2023twostage} and 38\% for \citet{mao2024regressionmultiexpertdeferral}. Expert involvement is minimized, with only 12\% and 22\% of queries allocated to Experts 1 and 2, respectively.

Since the experts possess distinct knowledge for the two tasks (\( C^{\text{M}_1}_{\text{cla}} \) and \( C^{\text{M}_1}_{\text{reg}} \) for \( \text{M}_1 \)), independently deferring classification and regression may lead to suboptimal performance. In contrast, our approach models deferral decisions dependently, considering the interplay between the two components to achieve better overall results.

\section{Conclusion}
We introduced a Two-Stage Learning-to-Defer framework for multi-task problems, extending existing approaches to jointly handle classification and regression. We proposed a two-stage surrogate loss family that is both \( (\mathcal{G}, \mathcal{R}) \)-consistent and Bayes-consistent for any cross-entropy-based surrogate. Additionally, we derived tight consistency bounds linked to cross-entropy losses and the \( L_1 \)-norm of aggregated costs. We further established novel minimizability gap for the two-stage setting, generalizing prior results to Learning-to-Defer with multiple experts. Finally, we showed that our learning bounds improve with a richer hypothesis space and more confident experts.

We validated our framework on two challenging tasks: (i) object detection, where classification and regression are inherently interdependent—beyond the scope of existing L2D methods; and (ii) electronic health record analysis, where we demonstrated that current L2D approaches can be suboptimal even when classification and regression tasks are independent.

\section*{Acknowledgment}
This research is supported by the National Research Foundation, Singapore under its AI
Singapore Programme (AISG Award No: AISG2-PhD-2023-01-041-J) and by A*STAR, and is part of the
programme DesCartes which is supported by the National Research Foundation, Prime Minister’s Office, Singapore under its Campus for Research Excellence and Technological Enterprise (CREATE) programme.

\section*{Impact Statement}
This paper advances the theoretical and practical understanding of machine learning, contributing to the development of more effective models and methods. While our research does not present any immediate or significant ethical concerns, we recognize the potential for indirect societal impacts. 


\bibliography{biblio}

\begin{thebibliography}{71}
\providecommand{\natexlab}[1]{#1}
\providecommand{\url}[1]{\texttt{#1}}
\expandafter\ifx\csname urlstyle\endcsname\relax
  \providecommand{\doi}[1]{doi: #1}\else
  \providecommand{\doi}{doi: \begingroup \urlstyle{rm}\Url}\fi

\bibitem[Awasthi et~al.(2022)Awasthi, Mao, Mohri, and Zhong]{Awasthi_Mao_Mohri_Zhong_2022_multi}
Awasthi, P., Mao, A., Mohri, M., and Zhong, Y.
\newblock Multi-class h-consistency bounds.
\newblock In \emph{Proceedings of the 36th International Conference on Neural Information Processing Systems}, NIPS '22, Red Hook, NY, USA, 2022. Curran Associates Inc.
\newblock ISBN 9781713871088.

\bibitem[Balagurunathan et~al.(2021)Balagurunathan, Mitchell, and El~Naqa]{balagurunathan2021requirements}
Balagurunathan, Y., Mitchell, R., and El~Naqa, I.
\newblock Requirements and reliability of {AI} in the medical context.
\newblock \emph{Physica Medica}, 83:\penalty0 72--78, 2021.

\bibitem[Bartlett et~al.(2006)Bartlett, Jordan, and McAuliffe]{bartlett1}
Bartlett, P., Jordan, M., and McAuliffe, J.
\newblock Convexity, classification, and risk bounds.
\newblock \emph{Journal of the American Statistical Association}, 101:\penalty0 138--156, 02 2006.
\newblock \doi{10.1198/016214505000000907}.

\bibitem[Bartlett \& Mendelson(2003)Bartlett and Mendelson]{Bartlett_rademecher}
Bartlett, P.~L. and Mendelson, S.
\newblock Rademacher and {Gaussian} complexities: Risk bounds and structural results.
\newblock \emph{J. Mach. Learn. Res.}, 3\penalty0 (null):\penalty0 463–482, March 2003.
\newblock ISSN 1532-4435.

\bibitem[Bartlett \& Wegkamp(2008)Bartlett and Wegkamp]{Bartlett_Wegkamp_2008}
Bartlett, P.~L. and Wegkamp, M.~H.
\newblock Classification with a reject option using a hinge loss.
\newblock \emph{Journal of Machine Learning Research}, 9\penalty0 (8), 2008.

\bibitem[Benz \& Rodriguez(2022)Benz and Rodriguez]{Benz}
Benz, N. L.~C. and Rodriguez, M.~G.
\newblock Counterfactual inference of second opinions.
\newblock In \emph{Uncertainty in Artificial Intelligence}, pp.\  453--463. PMLR, 2022.

\bibitem[Berkson(1944)]{Berkson1944ApplicationOT}
Berkson, J.
\newblock Application of the logistic function to bio-assay.
\newblock \emph{Journal of the American Statistical Association}, 39:\penalty0 357--365, 1944.
\newblock URL \url{https://api.semanticscholar.org/CorpusID:122893121}.

\bibitem[Buch et~al.(2017)Buch, Escorcia, Shen, Ghanem, and Carlos~Niebles]{buch2017sst}
Buch, S., Escorcia, V., Shen, C., Ghanem, B., and Carlos~Niebles, J.
\newblock {SST}: Single-stream temporal action proposals.
\newblock In \emph{Proceedings of the IEEE Conference on Computer Vision and Pattern Recognition}, pp.\  2911--2920, 2017.

\bibitem[Cao et~al.(2022)Cao, Cai, Feng, Gu, Gu, An, Niu, and Sugiyama]{Cao}
Cao, Y., Cai, T., Feng, L., Gu, L., Gu, J., An, B., Niu, G., and Sugiyama, M.
\newblock Generalizing consistent multi-class classification with rejection to be compatible with arbitrary losses.
\newblock In \emph{Proceedings of the 36th International Conference on Neural Information Processing Systems}, NIPS '22, Red Hook, NY, USA, 2022. Curran Associates Inc.
\newblock ISBN 9781713871088.

\bibitem[Cao et~al.(2024)Cao, Mozannar, Feng, Wei, and An]{Cao_Mozannar_Feng_Wei_An_2023}
Cao, Y., Mozannar, H., Feng, L., Wei, H., and An, B.
\newblock In defense of softmax parametrization for calibrated and consistent learning to defer.
\newblock In \emph{Proceedings of the 37th International Conference on Neural Information Processing Systems}, NIPS '23, Red Hook, NY, USA, 2024. Curran Associates Inc.

\bibitem[Carion et~al.(2020)Carion, Massa, Synnaeve, Usunier, Kirillov, and Zagoruyko]{detr}
Carion, N., Massa, F., Synnaeve, G., Usunier, N., Kirillov, A., and Zagoruyko, S.
\newblock End-to-end object detection with transformers.
\newblock In \emph{Computer Vision – ECCV 2020: 16th European Conference, Glasgow, UK, August 23–28, 2020, Proceedings, Part I}, pp.\  213–229, Berlin, Heidelberg, 2020. Springer-Verlag.
\newblock ISBN 978-3-030-58451-1.
\newblock \doi{10.1007/978-3-030-58452-8_13}.
\newblock URL \url{https://doi.org/10.1007/978-3-030-58452-8_13}.

\bibitem[Charusaie et~al.(2022)Charusaie, Mozannar, Sontag, and Samadi]{charusaie2022sample}
Charusaie, M.-A., Mozannar, H., Sontag, D., and Samadi, S.
\newblock Sample efficient learning of predictors that complement humans.
\newblock In \emph{International Conference on Machine Learning}, pp.\  2972--3005. PMLR, 2022.

\bibitem[Chow(2003)]{Chow_1970}
Chow, C.
\newblock On optimum recognition error and reject tradeoff.
\newblock \emph{IEEE Transactions on information theory}, 16\penalty0 (1):\penalty0 41--46, 2003.

\bibitem[Cortes \& Vapnik(1995)Cortes and Vapnik]{Cortes1995SupportVector}
Cortes, C. and Vapnik, V.~N.
\newblock Support-vector networks.
\newblock \emph{Machine Learning}, 20:\penalty0 273--297, 1995.
\newblock URL \url{https://api.semanticscholar.org/CorpusID:52874011}.

\bibitem[Cortes et~al.(2016)Cortes, DeSalvo, and Mohri]{cortes_rejection}
Cortes, C., DeSalvo, G., and Mohri, M.
\newblock Learning with rejection.
\newblock In \emph{Algorithmic Learning Theory: 27th International Conference, ALT 2016, Bari, Italy, October 19-21, 2016, Proceedings 27}, pp.\  67--82. Springer, 2016.

\bibitem[Cortes et~al.(2024)Cortes, Mao, Mohri, Mohri, and Zhong]{cortes2024cardinality}
Cortes, C., Mao, A., Mohri, C., Mohri, M., and Zhong, Y.
\newblock Cardinality-aware set prediction and top-$ k $ classification.
\newblock \emph{Advances in neural information processing systems}, 37:\penalty0 18265--18309, 2024.

\bibitem[Cortes et~al.(2025)Cortes, Mao, Mohri, and Zhong]{cortes2025balancing}
Cortes, C., Mao, A., Mohri, M., and Zhong, Y.
\newblock Balancing the scales: A theoretical and algorithmic framework for learning from imbalanced data.
\newblock In \emph{Forty-second International Conference on Machine Learning}, 2025.
\newblock URL \url{https://openreview.net/forum?id=gscscNNiPN}.

\bibitem[Everingham et~al.(2010)Everingham, Gool, Williams, Winn, and Zisserman]{pascal}
Everingham, M., Gool, L., Williams, C.~K., Winn, J., and Zisserman, A.
\newblock The pascal visual object classes (voc) challenge.
\newblock \emph{Int. J. Comput. Vision}, 88\penalty0 (2):\penalty0 303–338, June 2010.
\newblock ISSN 0920-5691.
\newblock \doi{10.1007/s11263-009-0275-4}.
\newblock URL \url{https://doi.org/10.1007/s11263-009-0275-4}.

\bibitem[Geifman \& El-Yaniv(2017)Geifman and El-Yaniv]{Geifman_El-Yaniv_2017}
Geifman, Y. and El-Yaniv, R.
\newblock Selective classification for deep neural networks.
\newblock In Guyon, I., Luxburg, U.~V., Bengio, S., Wallach, H., Fergus, R., Vishwanathan, S., and Garnett, R. (eds.), \emph{Advances in Neural Information Processing Systems}, volume~30. Curran Associates, Inc., 2017.
\newblock URL \url{https://proceedings.neurips.cc/paper_files/paper/2017/file/4a8423d5e91fda00bb7e46540e2b0cf1-Paper.pdf}.

\bibitem[Ghosh et~al.(2017)Ghosh, Kumar, and Sastry]{Ghosh2017RobustLF}
Ghosh, A., Kumar, H., and Sastry, P.~S.
\newblock Robust loss functions under label noise for deep neural networks.
\newblock \emph{ArXiv}, abs/1712.09482, 2017.
\newblock URL \url{https://api.semanticscholar.org/CorpusID:6546734}.

\bibitem[Girshick(2015)]{girshick2015fast}
Girshick, R.
\newblock Fast {R-CNN}.
\newblock \emph{arXiv preprint arXiv:1504.08083}, 2015.

\bibitem[Hemmer et~al.(2021)Hemmer, Schemmer, V{\"o}ssing, and K{\"u}hl]{hemmer2021human}
Hemmer, P., Schemmer, M., V{\"o}ssing, M., and K{\"u}hl, N.
\newblock Human-{AI} complementarity in hybrid intelligence systems: A structured literature review.
\newblock \emph{PACIS}, pp.\ ~78, 2021.

\bibitem[Hemmer et~al.(2022)Hemmer, Schellhammer, Vössing, Jakubik, and Satzger]{Hemmer}
Hemmer, P., Schellhammer, S., Vössing, M., Jakubik, J., and Satzger, G.
\newblock Forming effective human-{AI} teams: Building machine learning models that complement the capabilities of multiple experts.
\newblock In Raedt, L.~D. (ed.), \emph{Proceedings of the Thirty-First International Joint Conference on Artificial Intelligence, {IJCAI-22}}, pp.\  2478--2484. International Joint Conferences on Artificial Intelligence Organization, 7 2022.
\newblock \doi{10.24963/ijcai.2022/344}.
\newblock URL \url{https://doi.org/10.24963/ijcai.2022/344}.
\newblock Main Track.

\bibitem[Johnson et~al.(2016)Johnson, Pollard, Shen, Lehman, Feng, Ghassemi, Moody, Szolovits, Anthony~Celi, and Mark]{johnson2016mimic}
Johnson, A.~E., Pollard, T.~J., Shen, L., Lehman, L.-w.~H., Feng, M., Ghassemi, M., Moody, B., Szolovits, P., Anthony~Celi, L., and Mark, R.~G.
\newblock {MIMIC-III}, a freely accessible critical care database.
\newblock \emph{Scientific data}, 3\penalty0 (1):\penalty0 1--9, 2016.

\bibitem[Johnson et~al.(2023)Johnson, Bulgarelli, Shen, Gayles, Shammout, Horng, Pollard, Hao, Moody, Gow, et~al.]{johnson2023mimic}
Johnson, A.~E., Bulgarelli, L., Shen, L., Gayles, A., Shammout, A., Horng, S., Pollard, T.~J., Hao, S., Moody, B., Gow, B., et~al.
\newblock {MIMIC-IV}, a freely accessible electronic health record dataset.
\newblock \emph{Scientific data}, 10\penalty0 (1):\penalty0 1, 2023.

\bibitem[Kerrigan et~al.(2021)Kerrigan, Smyth, and Steyvers]{Kerrigan}
Kerrigan, G., Smyth, P., and Steyvers, M.
\newblock Combining human predictions with model probabilities via confusion matrices and calibration.
\newblock In Ranzato, M., Beygelzimer, A., Dauphin, Y., Liang, P., and Vaughan, J.~W. (eds.), \emph{Advances in Neural Information Processing Systems}, volume~34, pp.\  4421--4434. Curran Associates, Inc., 2021.
\newblock URL \url{https://proceedings.neurips.cc/paper_files/paper/2021/file/234b941e88b755b7a72a1c1dd5022f30-Paper.pdf}.

\bibitem[Keswani et~al.(2021)Keswani, Lease, and Kenthapadi]{Keswani}
Keswani, V., Lease, M., and Kenthapadi, K.
\newblock Towards unbiased and accurate deferral to multiple experts.
\newblock In \emph{Proceedings of the 2021 AAAI/ACM Conference on AI, Ethics, and Society}, AIES '21, pp.\  154–165, New York, NY, USA, 2021. Association for Computing Machinery.
\newblock ISBN 9781450384735.
\newblock \doi{10.1145/3461702.3462516}.
\newblock URL \url{https://doi.org/10.1145/3461702.3462516}.

\bibitem[Kingma \& Ba(2017)Kingma and Ba]{kingma2017adammethodstochasticoptimization}
Kingma, D.~P. and Ba, J.
\newblock Adam: A method for stochastic optimization, 2017.
\newblock URL \url{https://arxiv.org/abs/1412.6980}.

\bibitem[Liu et~al.(2024)Liu, Cao, Zhang, Feng, and An]{liu2024mitigating}
Liu, S., Cao, Y., Zhang, Q., Feng, L., and An, B.
\newblock Mitigating underfitting in learning to defer with consistent losses.
\newblock In \emph{International Conference on Artificial Intelligence and Statistics}, pp.\  4816--4824. PMLR, 2024.

\bibitem[Lloyd(1982)]{kmeans}
Lloyd, S.
\newblock Least squares quantization in {PCM}.
\newblock \emph{IEEE Transactions on Information Theory}, 28\penalty0 (2):\penalty0 129--137, 1982.
\newblock \doi{10.1109/TIT.1982.1056489}.

\bibitem[Long \& Servedio(2013)Long and Servedio]{pmlr-v28-long13}
Long, P. and Servedio, R.
\newblock Consistency versus realizable h-consistency for multiclass classification.
\newblock In Dasgupta, S. and McAllester, D. (eds.), \emph{Proceedings of the 30th International Conference on Machine Learning}, number~3 in Proceedings of Machine Learning Research, pp.\  801--809, Atlanta, Georgia, USA, 17--19 Jun 2013. PMLR.
\newblock URL \url{https://proceedings.mlr.press/v28/long13.html}.

\bibitem[Madras et~al.(2018)Madras, Pitassi, and Zemel]{madras2018predict}
Madras, D., Pitassi, T., and Zemel, R.
\newblock Predict responsibly: Improving fairness and accuracy by learning to defer.
\newblock In \emph{Proceedings of the 32nd International Conference on Neural Information Processing Systems}, NIPS'18, pp.\  6150–6160, 2018.

\bibitem[Mao et~al.(2023{\natexlab{a}})Mao, Mohri, Mohri, and Zhong]{mao2023twostage}
Mao, A., Mohri, C., Mohri, M., and Zhong, Y.
\newblock Two-stage learning to defer with multiple experts.
\newblock In \emph{Thirty-seventh Conference on Neural Information Processing Systems}, 2023{\natexlab{a}}.
\newblock URL \url{https://openreview.net/forum?id=GIlsH0T4b2}.

\bibitem[Mao et~al.(2023{\natexlab{b}})Mao, Mohri, and Zhong]{mao2023crossentropylossfunctionstheoretical}
Mao, A., Mohri, M., and Zhong, Y.
\newblock Cross-entropy loss functions: Theoretical analysis and applications.
\newblock In \emph{Proceedings of the 40th International Conference on Machine Learning}, ICML'23, 2023{\natexlab{b}}.

\bibitem[Mao et~al.(2024{\natexlab{a}})Mao, Mohri, and Zhong]{Mao_Mohri_Zhong_2023_abst}
Mao, A., Mohri, M., and Zhong, Y.
\newblock Predictor-rejector multi-class abstention: Theoretical analysis and algorithms.
\newblock In \emph{International Conference on Algorithmic Learning Theory}, pp.\  822--867. PMLR, 2024{\natexlab{a}}.

\bibitem[Mao et~al.(2024{\natexlab{b}})Mao, Mohri, and Zhong]{mao2024enhanced}
Mao, A., Mohri, M., and Zhong, Y.
\newblock Enhanced $ h $-consistency bounds.
\newblock \emph{arXiv preprint arXiv:2407.13722}, 2024{\natexlab{b}}.

\bibitem[Mao et~al.(2024{\natexlab{c}})Mao, Mohri, and Zhong]{mao2024h}
Mao, A., Mohri, M., and Zhong, Y.
\newblock H-consistency guarantees for regression.
\newblock \emph{Proceedings of Machine Learning Research}, 235:\penalty0 34712--34737, 2024{\natexlab{c}}.

\bibitem[Mao et~al.(2024{\natexlab{d}})Mao, Mohri, and Zhong]{mao2024multi}
Mao, A., Mohri, M., and Zhong, Y.
\newblock Multi-label learning with stronger consistency guarantees.
\newblock \emph{Advances in neural information processing systems}, 37:\penalty0 2378--2406, 2024{\natexlab{d}}.

\bibitem[Mao et~al.(2024{\natexlab{e}})Mao, Mohri, and Zhong]{mao2024principled}
Mao, A., Mohri, M., and Zhong, Y.
\newblock Principled approaches for learning to defer with multiple experts.
\newblock In \emph{International Workshop on Combinatorial Image Analysis}, pp.\  107--135. Springer, 2024{\natexlab{e}}.

\bibitem[Mao et~al.(2024{\natexlab{f}})Mao, Mohri, and Zhong]{mao2024realizablehconsistentbayesconsistentloss}
Mao, A., Mohri, M., and Zhong, Y.
\newblock Realizable \$h\$-consistent and bayes-consistent loss functions for learning to defer.
\newblock In \emph{The Thirty-eighth Annual Conference on Neural Information Processing Systems}, 2024{\natexlab{f}}.
\newblock URL \url{https://openreview.net/forum?id=OcO2XakUUK}.

\bibitem[Mao et~al.(2024{\natexlab{g}})Mao, Mohri, and Zhong]{mao2024regressionmultiexpertdeferral}
Mao, A., Mohri, M., and Zhong, Y.
\newblock Regression with multi-expert deferral.
\newblock arXiv 2403.19494, 2024{\natexlab{g}}.
\newblock https://arxiv.org/abs/2403.19494.

\bibitem[Mao et~al.(2024{\natexlab{h}})Mao, Mohri, and Zhong]{mao2024theoretically}
Mao, A., Mohri, M., and Zhong, Y.
\newblock Theoretically grounded loss functions and algorithms for score-based multi-class abstention.
\newblock In \emph{International Conference on Artificial Intelligence and Statistics}, pp.\  4753--4761. PMLR, 2024{\natexlab{h}}.

\bibitem[Mao et~al.(2024{\natexlab{i}})Mao, Mohri, and Zhong]{mao2024universal}
Mao, A., Mohri, M., and Zhong, Y.
\newblock A universal growth rate for learning with smooth surrogate losses.
\newblock \emph{Advances in neural information processing systems}, 37:\penalty0 41670--41708, 2024{\natexlab{i}}.

\bibitem[Mao et~al.(2025{\natexlab{a}})Mao, Mohri, and Zhong]{mao2025enhanced}
Mao, A., Mohri, M., and Zhong, Y.
\newblock Enhanced \$h\$-consistency bounds.
\newblock In \emph{36th International Conference on Algorithmic Learning Theory}, 2025{\natexlab{a}}.
\newblock URL \url{https://openreview.net/forum?id=qgnVGFJMJo}.

\bibitem[Mao et~al.(2025{\natexlab{b}})Mao, Mohri, and Zhong]{mao2025mastering}
Mao, A., Mohri, M., and Zhong, Y.
\newblock Mastering multiple-expert routing: Realizable \$h\$-consistency and strong guarantees for learning to defer.
\newblock In \emph{Forty-second International Conference on Machine Learning}, 2025{\natexlab{b}}.
\newblock URL \url{https://openreview.net/forum?id=2KlxjR6lsd}.

\bibitem[Mao et~al.(2025{\natexlab{c}})Mao, Mohri, and Zhong]{mao2025principled}
Mao, A., Mohri, M., and Zhong, Y.
\newblock Principled algorithms for optimizing generalized metrics in binary classification.
\newblock In \emph{Forty-second International Conference on Machine Learning}, 2025{\natexlab{c}}.
\newblock URL \url{https://openreview.net/forum?id=YngQelHz1X}.

\bibitem[Mohri et~al.(2012)Mohri, Rostamizadeh, and Talwalkar]{Mohri}
Mohri, M., Rostamizadeh, A., and Talwalkar, A.
\newblock \emph{Foundations of Machine Learning}.
\newblock The MIT Press, 2012.
\newblock ISBN 026201825X.

\bibitem[Montreuil et~al.(2024)Montreuil, Yeo, Carlier, Ng, and Ooi]{montreuil2024learningtodeferextractivequestionanswering}
Montreuil, Y., Yeo, S.~H., Carlier, A., Ng, L.~X., and Ooi, W.~T.
\newblock Optimal query allocation in extractive qa with llms: A learning-to-defer framework with theoretical guarantees, 2024.

\bibitem[Montreuil et~al.(2025{\natexlab{a}})Montreuil, Carlier, Ng, and Ooi]{montreuil2025adversarialrobustnesstwostagelearningtodefer}
Montreuil, Y., Carlier, A., Ng, L.~X., and Ooi, W.~T.
\newblock Adversarial robustness in two-stage learning-to-defer: Algorithms and guarantees.
\newblock In \emph{Forty-second International Conference on Machine Learning}, 2025{\natexlab{a}}.
\newblock URL \url{https://openreview.net/forum?id=h3KHwZCnxH}.

\bibitem[Montreuil et~al.(2025{\natexlab{b}})Montreuil, Carlier, Ng, and Ooi]{montreuil2025askaskktwostage}
Montreuil, Y., Carlier, A., Ng, L.~X., and Ooi, W.~T.
\newblock Why ask one when you can ask $k$? two-stage learning-to-defer to the top-$k$ experts, 2025{\natexlab{b}}.

\bibitem[Montreuil et~al.(2025{\natexlab{c}})Montreuil, Carlier, Ng, and Ooi]{montreuil2025onestagetopklearningtodeferscorebased}
Montreuil, Y., Carlier, A., Ng, L.~X., and Ooi, W.~T.
\newblock One-stage top-$k$ learning-to-defer: Score-based surrogates with theoretical guarantees, 2025{\natexlab{c}}.

\bibitem[Mozannar \& Sontag(2020)Mozannar and Sontag]{mozannar2021consistent}
Mozannar, H. and Sontag, D.
\newblock Consistent estimators for learning to defer to an expert.
\newblock In \emph{Proceedings of the 37th International Conference on Machine Learning}, ICML'20, 2020.

\bibitem[Mozannar et~al.(2023)Mozannar, Lang, Wei, Sattigeri, Das, and Sontag]{Mozannar2023WhoSP}
Mozannar, H., Lang, H., Wei, D., Sattigeri, P., Das, S., and Sontag, D.~A.
\newblock Who should predict? exact algorithms for learning to defer to humans.
\newblock In \emph{International Conference on Artificial Intelligence and Statistics}, 2023.
\newblock URL \url{https://api.semanticscholar.org/CorpusID:255941521}.

\bibitem[Narasimhan et~al.(2022)Narasimhan, Jitkrittum, Menon, Rawat, and Kumar]{narasimhan2022post}
Narasimhan, H., Jitkrittum, W., Menon, A.~K., Rawat, A., and Kumar, S.
\newblock Post-hoc estimators for learning to defer to an expert.
\newblock \emph{Advances in Neural Information Processing Systems}, 35:\penalty0 29292--29304, 2022.

\bibitem[Okati et~al.(2021)Okati, De, and Rodriguez]{okati2021differentiable}
Okati, N., De, A., and Rodriguez, M.
\newblock Differentiable learning under triage.
\newblock \emph{Advances in Neural Information Processing Systems}, 34:\penalty0 9140--9151, 2021.

\bibitem[Palomba et~al.(2024)Palomba, Pugnana, {\'A}lvarez, and Ruggieri]{palomba2024causal}
Palomba, F., Pugnana, A., {\'A}lvarez, J.~M., and Ruggieri, S.
\newblock A causal framework for evaluating deferring systems.
\newblock \emph{CoRR}, 2024.

\bibitem[Ramaswamy et~al.(2018)Ramaswamy, Tewari, and Agarwal]{Ramaswamy}
Ramaswamy, H.~G., Tewari, A., and Agarwal, S.
\newblock {Consistent algorithms for multiclass classification with an abstain option}.
\newblock \emph{Electronic Journal of Statistics}, 12\penalty0 (1):\penalty0 530 -- 554, 2018.
\newblock \doi{10.1214/17-EJS1388}.
\newblock URL \url{https://doi.org/10.1214/17-EJS1388}.

\bibitem[Redmon et~al.(2016)Redmon, Divvala, Girshick, and Farhadi]{redmon2016lookonceunifiedrealtime}
Redmon, J., Divvala, S., Girshick, R., and Farhadi, A.
\newblock You only look once: Unified, real-time object detection.
\newblock In \emph{2016 IEEE Conference on Computer Vision and Pattern Recognition (CVPR)}, pp.\  779--788, 2016.

\bibitem[Ren et~al.(2016)Ren, He, Girshick, and Sun]{ren2016fasterrcnnrealtimeobject}
Ren, S., He, K., Girshick, R., and Sun, J.
\newblock Faster r-cnn: Towards real-time object detection with region proposal networks, 2016.
\newblock URL \url{https://arxiv.org/abs/1506.01497}.

\bibitem[Steinwart(2007)]{Steinwart2007HowTC}
Steinwart, I.
\newblock How to compare different loss functions and their risks.
\newblock \emph{Constructive Approximation}, 26:\penalty0 225--287, 2007.
\newblock URL \url{https://api.semanticscholar.org/CorpusID:16660598}.

\bibitem[Tailor et~al.(2024)Tailor, Patra, Verma, Manggala, and Nalisnick]{Tailor}
Tailor, D., Patra, A., Verma, R., Manggala, P., and Nalisnick, E.
\newblock Learning to defer to a population: A meta-learning approach.
\newblock In Dasgupta, S., Mandt, S., and Li, Y. (eds.), \emph{Proceedings of The 27th International Conference on Artificial Intelligence and Statistics}, volume 238 of \emph{Proceedings of Machine Learning Research}, pp.\  3475--3483. PMLR, 02--04 May 2024.
\newblock URL \url{https://proceedings.mlr.press/v238/tailor24a.html}.

\bibitem[Tewari \& Bartlett(2007)Tewari and Bartlett]{tewari07a}
Tewari, A. and Bartlett, P.~L.
\newblock On the consistency of multiclass classification methods.
\newblock \emph{Journal of Machine Learning Research}, 8\penalty0 (36):\penalty0 1007--1025, 2007.
\newblock URL \url{http://jmlr.org/papers/v8/tewari07a.html}.

\bibitem[Thorndike(1953)]{Thorndike1953WhoBI}
Thorndike, R.~L.
\newblock Who belongs in the family?
\newblock \emph{Psychometrika}, 18:\penalty0 267--276, 1953.
\newblock URL \url{https://api.semanticscholar.org/CorpusID:120467216}.

\bibitem[Vaswani et~al.(2017)Vaswani, Shazeer, Parmar, Uszkoreit, Jones, Gomez, Kaiser, and Polosukhin]{trans}
Vaswani, A., Shazeer, N., Parmar, N., Uszkoreit, J., Jones, L., Gomez, A.~N., Kaiser, L.~u., and Polosukhin, I.
\newblock Attention is all you need.
\newblock In Guyon, I., Luxburg, U.~V., Bengio, S., Wallach, H., Fergus, R., Vishwanathan, S., and Garnett, R. (eds.), \emph{Advances in Neural Information Processing Systems}, volume~30. Curran Associates, Inc., 2017.
\newblock URL \url{https://proceedings.neurips.cc/paper_files/paper/2017/file/3f5ee243547dee91fbd053c1c4a845aa-Paper.pdf}.

\bibitem[Verma \& Nalisnick(2022)Verma and Nalisnick]{Verma_Nalisnick_2022}
Verma, R. and Nalisnick, E.
\newblock Calibrated learning to defer with one-vs-all classifiers.
\newblock In \emph{International Conference on Machine Learning}, pp.\  22184--22202. PMLR, 2022.

\bibitem[Verma et~al.(2023)Verma, Barrejon, and Nalisnick]{Verma2022LearningTD}
Verma, R., Barrejon, D., and Nalisnick, E.
\newblock Learning to defer to multiple experts: Consistent surrogate losses, confidence calibration, and conformal ensembles.
\newblock In \emph{International Conference on Artificial Intelligence and Statistics}, 2023.
\newblock URL \url{https://api.semanticscholar.org/CorpusID:253237048}.

\bibitem[Wei et~al.(2024)Wei, Cao, and Feng]{wei2024exploiting}
Wei, Z., Cao, Y., and Feng, L.
\newblock Exploiting human-ai dependence for learning to defer.
\newblock In \emph{Forty-first International Conference on Machine Learning}, 2024.

\bibitem[Yang et~al.(2023)Yang, Wu, Jiang, Lin, Gao, Danek, and Sun]{pyhealth2023yang}
Yang, C., Wu, Z., Jiang, P., Lin, Z., Gao, J., Danek, B., and Sun, J.
\newblock {PyHealth}: A deep learning toolkit for healthcare predictive modeling.
\newblock In \emph{Proceedings of the 27th ACM SIGKDD International Conference on Knowledge Discovery and Data Mining (KDD) 2023}, 2023.
\newblock URL \url{https://github.com/sunlabuiuc/PyHealth}.

\bibitem[Zhang \& Agarwal(2020)Zhang and Agarwal]{Zhang}
Zhang, M. and Agarwal, S.
\newblock Bayes consistency vs. h-consistency: The interplay between surrogate loss functions and the scoring function class.
\newblock In Larochelle, H., Ranzato, M., Hadsell, R., Balcan, M., and Lin, H. (eds.), \emph{Advances in Neural Information Processing Systems}, volume~33, pp.\  16927--16936. Curran Associates, Inc., 2020.
\newblock URL \url{https://proceedings.neurips.cc/paper_files/paper/2020/file/c4c28b367e14df88993ad475dedf6b77-Paper.pdf}.

\bibitem[Zhang(2002)]{Statistical}
Zhang, T.
\newblock Statistical behavior and consistency of classification methods based on convex risk minimization.
\newblock \emph{Annals of Statistics}, 32, 12 2002.
\newblock \doi{10.1214/aos/1079120130}.

\bibitem[Zhong(2025)]{zhong2025fundamental}
Zhong, Y.
\newblock \emph{Fundamental Novel Consistency Theory: H-Consistency Bounds}.
\newblock New York University, 2025.

\end{thebibliography}
\bibliographystyle{icml2025}

\newpage
\clearpage
\onecolumn
\appendix

\newpage

\section{Algorithm}\label{appendix:algo}

\begin{algorithm}[ht]
   \caption{Two-Stage Learning-to-Defer for Multi-Task Learning Algorithm}
   \label{alg:l2d}
\begin{algorithmic}
   \STATE {\bfseries Input:} Dataset $\{(x_k, y_k, t_k)\}_{k=1}^K$, multi-task model $g\in\mc{G}$, experts $m\in\mc{M}$, rejector $r\in\mc{R}$, number of epochs $\text{EPOCH}$, batch size $B$, learning rate $\eta$.
   \STATE {\bfseries Initialization:} Initialize rejector parameters $\theta$.
   \FOR{$i=1$ to $\text{EPOCH}$}
       \STATE Shuffle dataset $\{(x_k, y_k, t_k)\}_{k=1}^K$.
       \FOR{each mini-batch $\mathcal{B} \subset \{(x_k, y_k, t_k)\}_{k=1}^K$ of size $B$}
           \STATE Extract input-output pairs $z=(x, y, t) \in \mathcal{B}$.
           \STATE Query model $g(x)$ and experts $m(x)$. \hfill\COMMENT{Agents are pre-trained and fixed}
           \STATE Evaluate costs $c_0(g(x),z)$ and $c_{j>0}(m(x),z)$. \hfill\COMMENT{Compute task-specific costs}
           \STATE Compute rejector prediction $r(x)=\arg\max_{j\in\mathcal{A}}r(x,j)$. \hfill\COMMENT{Rejector decision}
           \STATE Compute surrogate deferral empirical risk $\widehat{\mathcal{E}}_{\Phi_{\text{def}}}$:
           \STATE \hspace{1em} $\widehat{\mathcal{E}}_{\Phi_{\text{def}}} = \frac{1}{B} \sum_{z \in \mathcal{B}} \Big[ \Phi_{\text{def}}(g,r,m,z) \Big]$. \hfill\COMMENT{Empirical risk computation}
           \STATE Update parameters $\theta$ using gradient descent:
           \STATE \hspace{1em} $\theta \leftarrow \theta - \eta \nabla_\theta \widehat{\mathcal{E}}_{\Phi_{\text{def}}}$. \hfill\COMMENT{Parameter update}
       \ENDFOR
   \ENDFOR
   \STATE \textbf{Return:} trained rejector model $r^\ast$.
\end{algorithmic}
\end{algorithm}

We will prove key lemmas and theorems stated in our main paper. 
\section{Proof of Lemma \ref{lemma_pointwise}} \label{proof_rejector}

We aim to prove Lemma \ref{lemma_pointwise}, which establishes the optimal deferral decision by minimizing the conditional risk.

By definition, the Bayes-optimal rejector \( r^B(x) \) minimizes the conditional risk \( \mathcal{C}_{\ell_{\text{def}}} \), given by:
\begin{equation}
    \mathcal{C}_{\ell_{\text{def}}}(g, r, x) = \mathbb{E}_{y,t|x}[\ell_{\text{def}}(g, r, m, z)].
\end{equation}
Expanding the expectation, we obtain:
\begin{equation}
    \mathcal{C}_{\ell_{\text{def}}}(g, r, x) = \mathbb{E}_{y,t|x} \left[ \sum_{j=0}^J c_j(g(x), m_j(x), z) 1_{r(x) = j} \right].
\end{equation}
Using the linearity of expectation, this simplifies to:
\begin{equation}
    \mathcal{C}_{\ell_{\text{def}}}(g, r, x) = \sum_{j=0}^J \mathbb{E}_{y,t|x} \left[ c_j(g(x), m_j(x), z) \right] 1_{r(x) = j}.
\end{equation}

Since we seek the rejector that minimizes the expected loss, the Bayes-conditional risk is given by:
\begin{equation}
    \mathcal{C}^B_{\ell_{\text{def}}}(\mc{G}, \mc{R}, x) = \inf_{g\in\mathcal{G}, r\in\mathcal{R}} \mathbb{E}_{y,t|x}[\ell_{\text{def}}(g, r, m, z)].
\end{equation}
Rewriting this expression, we obtain:
\begin{equation}
    \mathcal{C}^B_{\ell_{\text{def}}}(\mc{G}, \mc{R}, x) = \inf_{r\in\mathcal{R}} \mathbb{E}_{y,t|x} \left[\inf_{g\in\mathcal{G}} c_0(g(x), z) 1_{r(x) = 0} + \sum_{j=1}^J c_j(m_j(x), z) 1_{r(x) = j} \right].
\end{equation}
This leads to the following minimization problem:
\begin{equation}
    \mathcal{C}^B_{\ell_{\text{def}}}(\mc{G}, \mc{R}, x) = \min \left\{ \inf_{g\in\mathcal{G}} \mathbb{E}_{y,t|x} \left[c_0(g(x), z)\right], \min_{j\in[J]} \mathbb{E}_{y,t|x} \left[c_j(m_j(x), z)\right] \right\}.
\end{equation}

To simplify notation, we define:
\begin{equation}
    \overline{c}_j^\ast =
    \begin{cases}
        \inf_{g\in\mathcal{G}} \mathbb{E}_{y, t|x} [c_0(g(x),z)], & \text{if } j = 0, \\
        \mathbb{E}_{y, t|x} [c_j(m_j(x),z)], & \text{otherwise}.
    \end{cases}
\end{equation}
Thus, the Bayes-conditional risk simplifies to:
\begin{equation}
    \mathcal{C}^B_{\ell_{\text{def}}}(\mc{G}, \mc{R}, x) = \min_{j \in \mathcal{A}} \overline{c}_j^\ast.
\end{equation}
Since the rejector selects the decision with the lowest expected cost, the optimal rejector is given by:
\begin{equation}
r^B(x) =
\begin{cases}
    0, & \text{if } \displaystyle \inf_{g\in\mathcal{G}} \mathbb{E}_{y,t|x} [c_0(g(x), z)] \leq \min_{j \in [J]} \mathbb{E}_{y,t|x} [c_j(m_j(x), z)], \\
    j, & \text{otherwise}.
\end{cases}
\end{equation}
This completes the proof. \qed

\section{Proof Theorem \ref{theo:consistency}}
\label{proof_consistency}
Before proving the desired Theorem \ref{theo:consistency}, we will use the following Lemma \ref{lemma_h_consi} \citep{Awasthi_Mao_Mohri_Zhong_2022_multi, mao2024regressionmultiexpertdeferral}:
\begin{lemma}[$\mc{R}$-consistency bound] \label{lemma_h_consi} Assume that the following  $\mc{R}$-consistency bounds holds for $r \in \mc{R}$, and any distribution 
\begin{equation*} \label{eq:r_const_01}
    \mc{E}_{\ell_{01}}(r) - \mc{E}_{\ell_{01}}^*(\mc{R}) + \mc{U}_{\ell_{01}}(\mc{R}) \leq \Gamma^\nu( \mc{E}_{\Phi_{01}^\nu}(r) - \mc{E}_{\Phi_{01}^\nu}^*(\mc{R}) + \mc{U}_{\Phi_{01}^\nu}(\mc{R}) )
\end{equation*}
then for $p \in (p_0 \dots p_J) \in \Delta^{|\mc{A}|}$ and $x \in \mc{X}$, we get
\begin{equation*}
    \sum^J_{j=0} p_j 1_{r(x) \neq j} - \inf_{r \in \mc{R}}  \sum^J_{j=0} p_j 1_{r(x) \neq j} \leq \Gamma^\nu \Big( \sum^J_{j=0} p_j \Phi_{01}^\nu(r, x, j) - \inf_{r \in \mc{R}}  \sum^J_{j=0} p_j \Phi_{01}^\nu(r, x, j)  \Big)
\end{equation*}
\end{lemma}

\consistency*

\begin{proof} 
Let denote a cost for $j\in\mc{A}=\{0, \dots, J\}$:
\begin{equation*}
    \begin{aligned}
    \overline{c}_j^\ast = \begin{cases}
        \inf_{g\in\mc{G}} \mb{E}_{y, t|x} [c_0(g(x), z)] & \text{if } j = 0 \\ \mb{E}_{y, t|x} [c_j(m(x), z)]& \text{otherwise}
    \end{cases}     
    \end{aligned}
\end{equation*}
Using the change of variables and the Bayes-conditional risk introduced in the proof of Lemma \ref{lemma_pointwise} in Appendix \ref{proof_rejector}, we have:

\begin{equation}
    \begin{aligned}  
         \mc{C}_{\ell_{\text{def}}}^B(\mc{G},\mc{R},x) & = \min_{j\in\mc{A}}\overline{c}^\ast_j \\
         \mc{C}_{\ell_{\text{def}}}(g,r,x)  & =  \sum_{j=0}^J \mb{E}_{y,t|x}\Big[c_j(g(x),m_j(x),z)\Big]1_{r(x)=j}
        \end{aligned}
    \end{equation} 
We follow suit for our surrogate \(\Phi_{\text{def}}\) and derive its conditional risk and optimal conditional risk.
\begin{equation*}
    \begin{aligned}
        \mc{C}_{\Phi_{\text{def}}} & = \mb{E}_{y,t|x}  \Big[\sum^{J}_{j = 1} c_j( m(x), z) \Phi_{01}^\nu(r, x, 0) + \sum^{J}_{j = 1} \Big(c_0(g(x), z) +  \sum^{J}_{i=1} c_i(m_i(x), z) 1_{j\neq i} \Big) \Phi_{01}^\nu(r, x, j) \\
         \mc{C}_{\Phi_{\text{def}}}^* & = \inf_{r \in \mathcal{R}}  \mb{E}_{y,t|x} \Big[ \sum^{J}_{j = 1} c_j(g(x), m(x), z) \Phi_{01}^\nu(r, x, 0)   + \sum^{J}_{j = 1} [c_0(g(x), z) + \sum^{J}_{i=1} c_i(m_i(x), z) 1_{j\neq i} ] \Phi_{01}^\nu(r, x, j)\Big] 
    \end{aligned}
\end{equation*}
Let us define the function \( v(m(x), z) = \min_{j \in [J]} \overline{c}_j(m_j(x), z) \), where \( m_j(x) \) denotes the model's output and \( \overline{c}_j \) represents the corresponding cost function. Using this definition, the calibration gap is formulated as \( \Delta\mathcal{C}_{\ell_{\text{def}}} := \mathcal{C}_{\ell_{\text{def}}} - \mathcal{C}^B_{\ell_{\text{def}}} \), where \( \mathcal{C}_{\ell_{\text{def}}} \) represents the original calibration term and \( \mathcal{C}^B_{\ell_{\text{def}}} \) denotes the baseline calibration term. By construction, the calibration gap satisfies \( \Delta\mathcal{C}_{\ell_{\text{def}}} \geq 0 \), leveraging the risks derived in the preceding analysis.

\begin{equation*}
    \begin{aligned}
        \Delta \mc{C}_{\ell_{\text{def}}} & = \mb{E}_{y, t|x} \Big[\rho(g(x),z)1_{r(x) = 0} + \sum^{J}_{j=1}  \Big(\rho(m(x), z) + \beta_j\Big)1_{r(x) = j} \Big]  &&\\& - v(m(x), z) + \Big(v(m(x), z)  - \min_{j \in \mc{A}} \overline{c}^\ast_{j}(g(x), m(x), z) \Big)
    \end{aligned}
\end{equation*}
Let us consider \(\Delta \mathcal{C}_{\ell_{\text{def}}} = A_1 + A_2\), such that:
\begin{equation}
    \begin{aligned}
        A_1 & = \mb{E}_{y, t|x} \Big[1_{r(x) = 0}\rho(g(x),z) + \sum^{J}_{j=1} 1_{r(x) = j} \Big(\rho(m_j(x),z) +\beta_j\Big) \Big]  - v(m(x), z) \\
         A_2 & = \Big(v(m(x), z)  - \min_{j \in \mc{A}} \overline{c}_{j}(g(x), m(x), z) \Big)
    \end{aligned}
\end{equation}

By considering the properties of \(\min\), we also get the following inequality:

\begin{equation}\label{eq:inequality}
    \begin{aligned}
        v(m(x), z) - \min_{j \in \mc{A}} \overline{c}^\ast_{j}(g(x), m(x), z) \leq  \mb{E}_{y,t|x} [c_0(g(x), z)] - \inf_{g \in \mathcal{G}} \mb{E}_{y,t|x} [c_0(g(x), z)]
    \end{aligned}
\end{equation}
implying, 
\begin{equation}
    \Delta \mc{C}_{\ell_{\text{def}}}\leq A_1 + \overline{c}_0(g(x), z) - \overline{c}_0^\ast(g(x), z)
\end{equation}

We now select a distribution for our rejector. We first define \(\forall j \in \mathcal{A}\),

\begin{equation*}
        p_{0} = \frac{\sum_{j = 1}^J \overline{c}_j(m_j(x), z)}{J \sum_{j=0}^J \overline{c}_j(g(x),m_j(x),z)}
\end{equation*}
and
\begin{equation*}
        p_{j\in[J]} = \frac{
        \overline{c}_0(g(x), z) + \sum^{J}_{j \neq j'} \overline{c}_j' (m_j(x), z)
        }{J \sum_{j=0}^J \overline{c}_j(g(x),m_j(x),z)}
\end{equation*}
which can also be written as:
\begin{equation}
    p_j = \frac{\overline{\tau}_j}{\|\overline{\boldsymbol{\tau}}\|_1}
\end{equation}

Injecting the new distribution, we obtain the following:
\begin{equation}\label{eq:delta_surrogate}
    \begin{aligned}
        \Delta\mc{C}_{\Phi_{\text{def}}} = \|\overline{\boldsymbol{\tau}}\|_1\Big(\sum_{j=0}^J p_j\Phi_{01}^\nu(r, x, j) - \inf_{r\in\mc{R}}\sum_{j=0}^J p_j\Phi_{01}^\nu(r, x, j)\Big)
    \end{aligned}
\end{equation}
Now consider the first and last term of $\Delta \mc{C}_{\ell_{\text{def}}}$. Following the intermediate step for Lemma \ref{surr:defer}, we have:
\begin{equation*}
    \begin{aligned}
      A_1  & =  \mb{E}_{y,t|x}[c_0(g(x), z)]1_{r(x)=0} + \sum_{j=1}^J\mb{E}_{y,t|x}[c_j(m_j(x), z)]1_{r(x)=j} - v(m(x), z) 
        \\& = \mb{E}_{y,t|x}[c_0(g(x), z)]1_{r(x)=0} + \sum_{j=1}^J\mb{E}_{y,t|x}[c_j(m_j(x), z)]1_{r(x)=j} \\&- \inf_{r \in \mathcal{R}} \Big[ \mb{E}_{y,t|x}[c_0(g(x), z)]1_{r(x)=0} + \sum_{j=1}^J\mb{E}_{y,t|x}[c_j(m_j(x), z)]1_{r(x)=j} \Big]
        \\& =   \sum^{J}_{j=1} \overline{c}_j (z, m_j)  1_{r(x) \neq 0}  + \sum^{J}_{j = 1}\Big( \overline{c}_0(g(x), z)  + \sum^{J}_{j \neq j'} \overline{c}_{j'}(m_{j'}(x), z) \Big)1_{r(x) \neq j} 
        &&\\&-  \inf_{r\in\mathcal{R}}\Big[ \sum^{J}_{j=1} \overline{c}_j (m_{j'}(x), z) 1_{r(x) \neq 0} + \sum^{J}_{j = 1}\Big(\overline{c}_0(g(x), z) + \sum^{J}_{j \neq j'} \overline{c}_{j'}(m_{j'}(x), z)\Big)  1_{r(x) \neq j} \Big] 
  \end{aligned}
\end{equation*}
Then, applying a change of variables to introduce \( \|\overline{\boldsymbol{\tau}}\|_1 \), we get:
\begin{equation*}
    \begin{aligned}
        & \|\overline{\boldsymbol{\tau}}\|_1 p_0 1_{r(x) \neq 0} + \|\overline{\boldsymbol{\tau}}\|_1 \sum_{j = 1}^{J} p_j 1_{r(x) \neq j} - \inf_{r \in \mathcal{R}} [\|\overline{\boldsymbol{\tau}}\|_1 p_0 1_{r(x) \neq 0} + \|\overline{\boldsymbol{\tau}}\|_1 \sum_{j = 1}^{J} p_j 1_{r(x) \neq j}]
        &&\\& = \|\overline{\boldsymbol{\tau}}\|_1 \sum_{j = 0}^{J} p_j 1_{r(x) \neq j} - \inf_{r \in \mathcal{R}} \|\overline{\boldsymbol{\tau}}\|_1 \sum_{j = 0}^{J} p_j 1_{r(x) \neq j}
    \end{aligned}
\end{equation*}
We now apply Lemma \ref{lemma_h_consi} to introduce $\Gamma$,
\begin{equation}
    \label{eq:lemmafiveres}
    \begin{aligned}
         \sum_{j = 0}^{J} p_j 1_{r(x) \neq j} - \inf_{r \in \mathcal{R}} \sum_{j = 0}^{J} p_j 1_{r(x) \neq j}  & \leq \Gamma \Big( \sum_{j = 0}^{J} p_j \Phi_{01}^\nu(r,x, j) - \inf_{r \in \mathcal{R}}  \sum_{j = 0}^{J} p_j \Phi_{01}^\nu(r,x, j) \Big) \\
         \frac{1}{{\|\overline{\boldsymbol{\tau}}\|_1}} \Big[\sum_{j = 0}^{J} \overline{\tau}_j 1_{r(x) \neq j} - \inf_{r \in \mathcal{R}} \sum_{j = 0}^{J} \overline{\tau}_j 1_{r(x) \neq j}\Big] &  \leq  \Gamma\Big( \frac{1}{{\|\overline{\boldsymbol{\tau}}\|_1}}\Big[\sum_{j = 0}^{J} \overline{\tau}_j \Phi_{01}^\nu(r,x, j) - \inf_{r \in \mathcal{R}} \sum_{j = 0}^{J} \overline{\tau}_j \Phi_{01}^\nu(r,x, j)\Big]\Big) \\
        \Delta  \mc{C}_{\ell_{\text{def}}}  & \leq  \|\overline{\boldsymbol{\tau}}\|_1 \Gamma \Big(\frac{\Delta\mc{C}_{\Phi_{\text{def}}}}{\|\overline{\boldsymbol{\tau}}\|_1}\Big) 
    \end{aligned}
\end{equation}

We reintroduce the coefficient \( A_2 \) such that:
\begin{equation*}
    \begin{aligned}
    \Delta  \mc{C}_{\ell_{\text{def}} } &\leq \|\overline{\boldsymbol{\tau}}\|_1 \Gamma \Big(\frac{\Delta\mc{C}_{\Phi_{\text{def}}}}{\|\overline{\boldsymbol{\tau}}\|_1}\Big) + A_2 \\
        \Delta  \mc{C}_{\ell_{\text{def}} } &\leq \|\overline{\boldsymbol{\tau}}\|_1 \Gamma \Big(\frac{\Delta\mc{C}_{\Phi_{\text{def}}}}{\|\overline{\boldsymbol{\tau}}\|_1}\Big) + \mb{E}_{y,t|x} [c_0(g(x), z)] - \inf_{g\in\mc{G}} \mb{E}_{y,t|x} [c_0(g(x), z)] \quad \text{(upper bounding with Eq \ref{eq:inequality})}
    \end{aligned}
\end{equation*}
\citet{mao2023crossentropylossfunctionstheoretical} introduced a tight bound for the comp-sum surrogates family. It follows for $\nu\geq0$ the inverse transformation $\Gamma^{\nu}(u) = \mathcal{T}^{-1, \nu}(u)$:

\[
\mathcal{T}^{\nu}(v) =
\begin{cases}
\frac{2^{1-\nu}}{1-\nu} \left[ 1 - \left( \frac{(1+v)^{\frac{2-\nu}{2}} + (1-v)^{\frac{2-\nu}{2}}}{2} \right)^{2-\nu} \right] & \nu \in [0,1) \\[12pt]
\frac{1+v}{2} \log[1+v] + \frac{1-v}{2} \log[1-v] & \nu = 1 \\[12pt]
\frac{1}{(\nu-1)n^{\nu-1}} \left[ \left( \frac{(1+v)^{\frac{2-\nu}{2}} + (1-v)^{\frac{2-\nu}{2}}}{2} \right)^{2-\nu} -1 \right] & \nu \in (1,2) \\[12pt]
\frac{1}{(\nu-1)n^{\nu-1}} v & \nu \in [2,+\infty).
\end{cases}
\]

We note $\overline{\Gamma}^\nu(u)=\|\overline{\boldsymbol{\tau}}\|_1 \Gamma^\nu(\frac{u}{\|\overline{\boldsymbol{\tau}}\|_1})$. 
By applying Jensen's Inequality and taking expectation on both sides, we get
\begin{align*}
    &\mathcal{E}_{\ell_{\text{def}}}(g, r) - \mathcal{E}^B_{\ell_{\text{def}}}(\mathcal{G}, \mathcal{R}) + \mathcal{U}_{\ell_{\text{def}}}(\mathcal{G}, \mathcal{R}) \\ &\leq \overline{\Gamma}^\nu(\mathcal{E}_{\Phi_{\text{def}}}(r) - \mathcal{E}^*_{\Phi_{\text{def}}}(\mathcal{R}) + \mathcal{U}_{\Phi_{\text{def}}}(\mathcal{R})) + \mathcal{E}_{c_0}(g) - \mathcal{E}^B_{c_0}(\mathcal{G}) + \mathcal{U}_{c_0}(\mathcal{G})
\end{align*}

\end{proof}

\section{Proof Theorem \ref{minimizability}} \label{proof_minimi}

\minimizability*
\begin{proof}
    
We define the softmax distribution as \( s_j = \frac{e^{r(x,j)}}{\sum_{j'\in\mathcal{A}} e^{r(x,j')}} \), where \( s_j \in [0,1] \). Let \( \overline{\tau}_j = \overline{\tau}_j(g(x), m(x), z) \) with \( \tau_j \in \mathbb{R}^+ \), and denote the expected value as \( \overline{\tau} = \mathbb{E}_{y,t|x}[\tau] \). We now derive the conditional risk for a given \( \nu \geq 0 \):

\begin{equation}
\begin{aligned}
        \mc{C}_{\Phi_\text{def}}^{\nu}(r,x) & = \sum_{j=0}^J \mb{E}_{y,t|x}[\tau_j]\Phi_{01}^\nu(r,x,j) \\
        & = \begin{cases}
        \frac{1}{1-\nu}\sum_{j=0}^J \overline{\tau}_j\Big[\Big(\sum_{j'\in\mc{A}} e^{r(x,j')-r(x,j)}\Big)^{1-\nu} - 1 \Big] & \nu\not=1 \\
        \sum_{j=0}^J \overline{\tau}_j\log\Big(\sum_{j'\in\mc{A}} e^{r(x,j')-r(x,j)}\Big) & \nu=1
    \end{cases} \\
    & = \begin{cases}
        \frac{1}{1-\nu}\sum_{j=0}^J \overline{\tau}_j\Big[s_j^{\nu-1} -1 \Big] & \nu\not=1 \\
        - \sum_{j=0}^J \overline{\tau}_j\log(s_j) & \nu=1
    \end{cases}
\end{aligned}
\end{equation}
\paragraph{For $\nu=1$:} we can write the following conditional risk:

\begin{equation}
    \begin{aligned}
        \mc{C}_{\Phi_\text{def}}^{\nu=1}(r,x) = -\sum_{j=0}^J \overline{\tau}_j\Big[r(x,j) - \log\sum_{j'\in\mc{A}}e^{r(x,j')}\Big] 
    \end{aligned}
\end{equation}
Then, 

\begin{equation}
    \begin{aligned}
        \frac{\partial \mc{C}_{\Phi_\text{def}}^{\nu=1}}{\partial r(x,i)}(r,x) = -\overline{\tau}_i +  \Big(\sum_{j=0}^J\overline{\tau}_j \Big) s_i^\ast
    \end{aligned}
\end{equation}
At the optimum, we have:
\begin{equation}
    s^\ast(x,i) = \frac{\overline{\tau}_i}{\sum_{j=0}\overline{\tau}_j}
\end{equation}
Then, it follows:
\begin{equation}
    \begin{aligned}
        \mc{C}_{\Phi_\text{def}}^{\ast, \nu=1}(\mc{R},x) & = -\sum_{j=0}^J \overline{\tau}_j\log\Big(\frac{\overline{\tau}_j}{\sum_{j'=0}\overline{\tau}_{j'}}\Big)
    \end{aligned}
\end{equation}

As the softmax parametrization is a distribution $s^\ast \in \Delta^{|\mc{A}|}$, we can write this conditional in terms of entropy with $\boldsymbol{\overline{\tau}} = \{\overline{\tau}_j\}_{j\in\mc{A}}$:
\begin{equation}
    \begin{aligned}
        \mc{C}_{\Phi_\text{def}}^{\ast, \nu=1}(\mc{R},x) & = -\Big(\sum_{k=0}^J \overline{\tau}_k\Big) \sum_{j=0}s^\ast_j\log(s^\ast_j) \\
        & = \Big(\sum_{k=0}^J \overline{\tau}_k\Big) H\Big(\frac{\overline{\tau}}{\sum_{j'=0}\overline{\tau}_{j'}}\Big) \\
        & = \|\overline{\boldsymbol{\tau}}\|_1 H\Big(\frac{\overline{\boldsymbol{\tau}}}{\|\overline{\boldsymbol{\tau}}\|_1}\Big) \quad \text{(as $\tau_j \in \mb{R}^+$)}
    \end{aligned}
\end{equation}

\paragraph{For $\nu\not=1,2$:} The softmax parametrization can be written as a constraint $\sum_{j=0}^J s_j = 1$ and $s_j \ge 0$. Consider the objective
\begin{equation}
\label{eq:tsallis-loss}
\Phi(\mathbf{s})
\;=\;
\frac{1}{1-\nu}\sum_{j=0}^J \overline{\tau}_j \,\,\Bigl[\,s_j^{\,\nu-1} - 1\Bigr].
\end{equation}
We aim to find $\mathbf{s}^* = \bigl(s_0^*,\dots,s_J^*\bigr)$ that minimizes \eqref{eq:tsallis-loss} subject to $\sum_{j=0}^J s_j = 1$.  Introduce a Lagrange multiplier $\lambda$ for the normalization $\sum_{j=0}^J s_j = 1$. The Lagrangian is:

\begin{equation}
\mathcal{L}(\mathbf{s}, \lambda)
\;=\;
\frac{1}{1-\nu}\,\sum_{j=0}^J \overline{\tau}_j\bigl[s_j^{\,\nu-1}-1\bigr]
\;+\;
\lambda \,\Bigl(1 - \sum_{j=0}^J s_j\Bigr).
\end{equation}
We take partial derivatives with respect to $s_i$:
\begin{equation}
\frac{\partial \mathcal{L}}{\partial s_i}
\;=\;
\frac{1}{\,1-\nu\,}\,\overline{\tau}_i\,(\nu-1)\,s_i^{\,\nu-2}
\;-\;
\lambda
\;=\;
0.
\end{equation}
Since $\frac{\nu-1}{\,1-\nu\,} = -1$, we get
\begin{equation}
\overline{\tau}_i\,s_i^{\,\nu-2}
\;=\;
-\lambda
\;\;>\;0
\;\;\Longrightarrow\;\;
s_i^{\,\nu-2}
\;=\;
\frac{\alpha}{\,\overline{\tau}_i\,}
\;\;\text{for some }\alpha>0.
\end{equation}
Hence
\begin{equation}
s_i
\;=\;
\Bigl(\tfrac{\alpha}{\overline{\tau}_i}\Bigr)^{\!\tfrac{1}{\,\nu-2\,}}.
\end{equation}
Summing $s_i$ over $\{i=0,\dots,J\}$ and setting the total to $1$ yields:
\begin{equation}
\sum_{i=0}^J 
\Bigl(\tfrac{\alpha}{\overline{\tau}_i}\Bigr)^{\!\tfrac{1}{\,\nu-2\,}}
\;=\;
1.
\end{equation}
Let 
\begin{equation}
\alpha^{\,\tfrac{1}{\,\nu-2\,}} 
\;=\;
\frac{1}{\sum_{k=0}^J (\tfrac{1}{\overline{\tau}_k})^{\tfrac{1}{\,\nu-2\,}}}
\;\;\Longrightarrow\;\;
\alpha 
\;=\;
\Bigl[\sum_{k=0}^J \bigl(\tfrac{1}{\overline{\tau}_k}\bigr)^{\tfrac{1}{\,\nu-2\,}}\Bigr]^{\!\nu-2}.
\end{equation}
Therefore, for each $i$,
\begin{equation}
s_i^*
\;=\;
\Bigl(\tfrac{\alpha}{\overline{\tau}_i}\Bigr)^{\!\tfrac{1}{\,\nu-2\,}}
\;=\;
\frac{\overline{\tau}_i^{\,\tfrac{1}{\,2-\nu\,}}}
{\displaystyle \sum_{k=0}^J \overline{\tau}_k^{\,\tfrac{1}{\,2-\nu\,}}}.
\end{equation}
This $\{s_i^*\}$ is a valid probability distribution. Let 
\begin{equation}
A 
\;=\;
\sum_{k=0}^J 
\tau_{k}^{\,\tfrac{1}{\,2-\nu\,}}.
\end{equation}
Then the optimum distribution  is
\begin{equation}
s_i^*
\;=\;
\frac{\overline{\tau}_i^{\,\tfrac{1}{\,2-\nu\,}}}{\,A\,}.
\end{equation}
Recall
\begin{equation}
\Phi(\mathbf{s})
\;=\;
\frac{1}{\,1-\nu\,}
  \sum_{j=0}^J \overline{\tau}_j\,\Big[ s_j^{\nu-1} -1 \Big].
\end{equation}
At \(s_j^*\), we have
\begin{equation}
(s_j^*)^{\nu-1}
\;=\;
\bigl(\tfrac{\overline{\tau}_j^{\tfrac{1}{\,2-\nu\,}}}{\,A\,}\bigr)^{\nu-1}
\;=\;
\frac{\overline{\tau}_j^{\,\tfrac{\nu-1}{\,2-\nu\,}}}{\,A^{\nu-1}\!}.
\end{equation}
Hence
\begin{equation}
\sum_{j=0}^J 
\overline{\tau}_j\,\bigl(s_j^*\bigr)^{\nu-1}
\;=\;
\frac{1}{\,A^{\nu-1}\!}
\;\sum_{j=0}^J 
\overline{\tau}_j^{\,1 + \tfrac{\nu-1}{\,2-\nu\,}}
\;=\;
\frac{1}{\,A^{\nu-1}\!}
\;\sum_{j=0}^J
\overline{\tau}_j^{\,\tfrac{1}{\,2-\nu\,}} 
\;=\;
\frac{A}{\,A^{\nu-1}\!}
\;=\;
A^{\,2-\nu}.
\end{equation}
Substituting back,
\begin{equation}
    \mc{C}_{\Phi_\text{def}}^{\ast, \nu\not=1,2}(\mc{R},x) = \frac{1}{\,1-\nu\,}\,\Bigl[
  \Bigl(\sum_{k=0}^J \overline{\tau}_k^{\,\tfrac{1}{\,2-\nu\,}}\Bigr)^{2-\nu} - \sum_{j=0}^J \overline{\tau}_j
\Bigr]
\end{equation}
We can express this conditional risk with a valid $L^{(\frac{1}{2-\nu})}$ norm as long as $\nu \in (1, 2)$. 
\begin{equation}
    \mc{C}_{\Phi_\text{def}}^{\ast, \nu\not=1,2}(\mc{R},x) = \frac{1}{\,\nu-1\,}\,\Bigl[ \|\overline{\boldsymbol{\tau}}\|_1 - \|\overline{\boldsymbol{\tau}}\|_{\frac{1}{2-\nu}} 
\Bigr]
\end{equation}

\paragraph{For $\nu=2$:}
Since \(\sum_{j=0}^J \overline{\tau}_j = S\), we have
\begin{equation}
\mc{C}_{\Phi_{\text{def}}}^{\nu=2}(r,x)
\;=\;
\sum_{j=0}^J \overline{\tau}_j\,\bigl[1 - s_j(r)\bigr]
\;=\;
\sum_{j=0}^J \overline{\tau}_j 
\;-\;
\sum_{j=0}^J \overline{\tau}_j\,s_j(r).
\end{equation}
Hence
\begin{equation}
\inf_{r\in\mc{R}}\,\mc{C}_{\Phi_{\text{def}}}^{\nu=2}(r,x)
\;=\;
S 
\;-\;
\sup_{r\in\mc{R}}\;\sum_{j=0}^J \overline{\tau}_j\,s_j(r).
\end{equation}
Therefore, minimizing \(\mc{C}_{\Phi_{\text{def}}}^{\nu=2}(r,x)\) is equivalent to maximizing
\begin{equation}
F(r)
\;=\;
\sum_{j=0}^J \overline{\tau}_j\,s_j(r).
\end{equation}

Its partial derivative w.r.t.\ \(r_i\) is the standard softmax derivative:
\begin{equation}
\frac{\partial s_j}{\partial r_i}
\;=\;
s_j \,\bigl(\delta_{ij} - s_i\bigr)
\;=\;
\begin{cases}
s_i\,(1 - s_i), & \text{if } i=j, \\
-\;s_j\,s_i, & \text{otherwise}.
\end{cases}
\end{equation}
Hence, for each \(i\),
\begin{equation}
\frac{\partial F}{\partial r_i}
\;=\;
\sum_{j=0}^J 
\overline{\tau}_j \,\frac{\partial s_j}{\partial r_i}
\;=\;
\overline{\tau}_i\,s_i\,(1 - s_i)
\;+\;
\sum_{\substack{j=0\\j\neq i}}^J \overline{\tau}_j\,\bigl(-s_j\,s_i\bigr).
\end{equation}

Factor out \(s_i\):
\begin{equation}
\frac{\partial F}{\partial r_i}
\;=\;
s_i\;\Bigl[\overline{\tau}_i\,(1 - s_i)
\;-\;
\sum_{j\neq i}\overline{\tau}_j\,s_j\Bigr]
\;=\;
s_i\;\Bigl[\overline{\tau}_i - \Bigl(\sum_{j=0}^J \overline{\tau}_j\,s_j\Bigr)\Bigr],
\end{equation}
because \(\sum_{j\neq i}\overline{\tau}_j\,s_j = \sum_{j=0}^J \overline{\tau}_j\,s_j - \overline{\tau}_i\,s_i\). Define \(F(r) = \sum_{j=0}^J \overline{\tau}_j\,s_j(r)\). Then:
\begin{equation}
\frac{\partial F}{\partial r_i}
\;=\;
s_i\,[\,\overline{\tau}_i - F(r)].
\end{equation}
Setting \(\tfrac{\partial F}{\partial r_i}=0\) for each \(i\) implies
\begin{equation}
s_i\,[\,\overline{\tau}_i - F(r)] 
\;=\; 
0,
\quad
\forall\,i.
\end{equation}
Thus, for each index \(i\):
\begin{equation}
s_i = 0
\quad\text{or}\quad
\overline{\tau}_i = F(r).
\end{equation}
To maximize \(F(r)\), notice that:
\begin{itemize}
\item If \(\overline{\tau}_{i^*}\) is strictly the largest among all \(\overline{\tau}_i\), then the maximum is approached by making \(s_{i^*}\approx 1\), so \(F(r)\approx \overline{\tau}_{i^*}\). In the softmax parameterization, this occurs in the limit \(r_{i^*}\to +\infty\) and \(r_k\to -\infty\) for \(k\neq i^*\).  
\item If there is a tie for the largest \(\overline{\tau}_i\), we can put mass on those coordinates that share the maximum value. In any case, the supremum is \(\max_i \overline{\tau}_i\).
\end{itemize}
Hence
\begin{equation}
\sup_{r\in\mc{R}} \;F(r)
\;=\;
\max_{0\le i\le J} \,\overline{\tau}_i.
\end{equation}
Because \(\mc{C}_{\Phi_{\text{def}}}^{\nu=2}(r,x) = S - F(r)\),
\begin{equation}
\inf_{r\in\mc{R}}\,\mc{C}_{\Phi_{\text{def}}}^{\nu=2}(r,x)
\;=\;
S \;-\; \sup_{r\in\mc{R}}\,F(r)
\;=\;
\sum_{j=0}^J \overline{\tau}_j \;-\; \max_{i \in \mc{A}}\,\overline{\tau}_i = \|\overline{\boldsymbol{\tau}}\|_1 - \|\overline{\boldsymbol{\tau}}\|_{\infty}
\end{equation}
Hence the global minimum of \(\mc{C}_{\Phi_{\text{def}}}^{\nu=2}\) is \(\|\overline{\boldsymbol{\tau}}\|_1 - \|\overline{\boldsymbol{\tau}}\|_{\infty}\). 
In the ``softmax'' parameterization, this is only approached in the limit as one coordinate \(r_{i^*}\) goes to \(+\infty\) and all others go to \(-\infty\). 
No finite \(r\) yields an exactly one-hot \(s_i(r)=1\), but the limit is enough to achieve the infimum arbitrarily closely.

It follows for $\overline{\boldsymbol{\tau}}=\{\overline{\tau}_j\}_{j\in\mc{A}}$ and $\nu\geq 0$:
\begin{equation}
\begin{aligned}
    \inf_{r\in\mc{R}}\,\mc{C}_{\Phi_{\text{def}}}^{\nu}(r,x) & = \begin{cases}
        \|\overline{\boldsymbol{\tau}}\|_1 H\Big(\frac{\overline{\boldsymbol{\tau}}}{\|\overline{\boldsymbol{\tau}}\|_1}\Big) & \nu=1 \\
        \|\overline{\boldsymbol{\tau}}\|_1 - \|\overline{\boldsymbol{\tau}}\|_{\infty} & \nu=2 \\
        \frac{1}{\,\nu-1\,}\,\Bigl[ \|\overline{\boldsymbol{\tau}}\|_1 - \|\overline{\boldsymbol{\tau}}\|_{\frac{1}{2-\nu}} 
        \Bigr] & \nu \in (1,2) \\
        \frac{1}{\,1-\nu\,}\,\Bigl[
        \Bigl(\sum_{k=0}^J \overline{\tau}_k^{\,\tfrac{1}{\,2-\nu\,}}\Bigr)^{2-\nu} -\|\overline{\boldsymbol{\tau}}\|_1
        \Bigr] &   \text{otherwise}
    \end{cases} \\
\end{aligned}
\end{equation}

Building on this, we can infer the minimizability gap:

\begin{equation}
    \begin{aligned}
        \mc{U}_{\Phi_{\text{def}}}(\mc{R}) = \mc{E}^\ast_{\Phi_{\text{def}}}(\mc{R})  -\mb{E}_{x}[\inf_{r\in\mc{R}}\,\mc{C}_{\Phi_{\text{def}}}^{\nu}(r,x)] 
    \end{aligned}
\end{equation}

\end{proof}

\section{Proof Lemma \ref{lemma_learning}} \label{proof_lemma}
\lemmalearning*

\begin{proof}
    We define the function $\psi$ as follows:
    \begin{equation}
        \psi :  \begin{array}{ccc}
                \mathcal{L}_1 + \mathcal{L}_2 & \longrightarrow & \mc{L}_1\mc{L}_2\\
                l_1 + l_2 & \longmapsto & (l_1 + l_2 - 1)_+
        \end{array}
    \end{equation}
    Here, $l_1 \in \mathcal{L}_1$ and $l_2 \in \mathcal{L}_2$. The function $\psi$ is 1-Lipschitz as we have $t \mapsto (t-1)_+$ for $t=l_1 + l_2$. Furthermore, given that $\psi$ is surjective and 1-Lipschitz, by Talagrand's lemma \citep{Mohri}, we have:
    \begin{equation}
        \hat{\mathfrak{R}}_S(\psi(\mc{L}_1 + \mc{L}_2)) \leq \hat{\mathfrak{R}}_S(\mc{L}_1 + \mc{L}_2) \leq \hat{\mathfrak{R}}_S(\mc{L}_1)  + \hat{\mathfrak{R}}_S(\mc{L}_2)
    \end{equation}
    This inequality shows that the Rademacher complexity of the sum of the losses is bounded by the sum of their individual complexities. 
\end{proof}

\section{Proof Theorem \ref{learning_bounds}}\label{proof_learning}
\learningbounds*
\begin{proof}
    
We are interested in finding the generalization of \( u = (g, r) \in \mathcal{L} \):
\begin{equation*}
    \begin{aligned}
        \mathfrak{R}_S(\mc{L}) & = \frac{1}{K}\mb{E}_{\sigma}[\sup_{g\in\mc{L}}\sum_{k=1}^K \sigma_k \ell_{\text{def}}(g, r, x_k, y_k, b_k, m_{k})] \\
        & = \frac{1}{K}\mb{E}_{\sigma}[\sup_{g\in\mc{L}}\sum_{k=1}^K \sigma_k \Big( \sum_{j=0}^J c_j1_{r(x_k)=j}\Big)] \\
        & \leq \frac{1}{K}\mb{E}_{\sigma}\Big[\sup_{g\in\mc{L}}\sum_{k=1}^K \sigma_k  c_01_{r(x_k)=0}\Big]   + \frac{1}{K}\sum_{j=1}^J\mb{E}_{\sigma}\Big[\sup_{r\in\mc{R}}\sum_{k=1}^K \sigma_k c_j1_{r(x_k)=j}\Big] \quad \text{(By the subadditivity of $\sup$)}
    \end{aligned}
\end{equation*}
Let's consider $j=0$:
\begin{equation}
    \begin{aligned}
        \frac{1}{K}\mb{E}_{\sigma}\Big[\sup_{g\in\mc{L}}\sum_{k=1}^K \sigma_k  c_01_{r(x_k)=0}\Big] & = \frac{1}{K}\mb{E}_{\sigma}\Big[\sup_{g\in\mc{L}}\sum_{k=1}^K \sigma_k  [1_{h(x_k)\not=y} + \ell_{\text{reg}}(f(x_k), b_k)] 1_{r(x_k)=0}\Big] \\
        & \leq \frac{1}{K}\mb{E}_{\sigma}\Big[\sup_{g\in\mc{L}}\sum_{k=1}^K \sigma_k  1_{h(x_k)\not=y}1_{r(x_k)=0}\Big] + \frac{1}{K}\mb{E}_{\sigma}\Big[\sup_{g\in\mc{L}}\sum_{k=1}^K \sigma_k\ell_{\text{reg}}(f(x_k), b_k)1_{r(x_k)=0}] \Big] \\
        & \leq \Big[\frac{1}{2}\mathfrak{R}_K(\mc{H}) + \mathfrak{R}_K(\mc{R})\Big] + \Big[ \mathfrak{R}_K(\mc{F}) + \mathfrak{R}_K(\mc{R}) \Big] \quad \text{(using Lemma \ref{lemma_learning})} \\ 
        & = \frac{1}{2}\mathfrak{R}_K(\mc{H}) + \mathfrak{R}_K(\mc{F}) + 2\mathfrak{R}_K(\mc{R})
    \end{aligned}
\end{equation}

Let's consider $j>0$: 
\begin{equation}
    \begin{aligned}
        \frac{1}{K}\sum_{j=1}^J\mb{E}_{\sigma}\Big[\sup_{r\in\mc{R}}\sum_{k=1}^K \sigma_k c_j1_{r(x_k)=j}\Big] & \leq \frac{1}{K}\sum_{j=1}^J\mb{E}_{\sigma}\Big[\sup_{r\in\mc{R}}\sum_{k=1}^K \sigma_k 1_{m_{k,j}^h\not=y}1_{r(x_k)=j}\Big] \\
        & + \frac{1}{K}\sum_{j=1}^J\mb{E}_{\sigma}\Big[\sup_{r\in\mc{R}}\sum_{k=1}^K \sigma_k \ell_{\text{reg}}(m_{k,j}^f, b_k)1_{r(x_k)=j}\Big] \\
    \end{aligned}
\end{equation}
Using learning-bounds for single expert in classification \citep{mozannar2021consistent}, we have:
\begin{equation}
    \begin{aligned}
        \frac{1}{K}\mb{E}_{\sigma}\Big[\sup_{r\in\mc{R}}\sum_{k=1}^K \sigma_k 1_{m_{k}^h\not=y}1_{r(x_k)=1}\Big] & \leq  \frac{\mc{D}(m^h \neq y)}{2} \exp\left(-\frac{K \mc{D}(m^h \neq y)}{8}\right) + \mathcal{R}_{K \mc{D}(m^h \neq y)/2}(\mathcal{R})
    \end{aligned}
\end{equation}
Applying it to our case:
\begin{equation}
    \begin{aligned}
        \frac{1}{K}\sum_{j=1}^J\mb{E}_{\sigma}\Big[\sup_{r\in\mc{R}}\sum_{k=1}^K \sigma_k 1_{m_{k,j}^h\not=y}1_{r(x_k)=j}\Big] & \leq  \sum_{j=1}^J \Big(\frac{\mc{D}(m_j^h \neq y)}{2} \exp\left(-\frac{K \mc{D}(m_j^h \neq y)}{8}\right) + \mathcal{R}_{K \mc{D}(m_j^h \neq y)/2}(\mathcal{R})\Big)
    \end{aligned}
\end{equation}
For the last term,
\begin{equation}
    \begin{aligned}
        \frac{1}{K}\sum_{j=1}^J\mb{E}_{\sigma}\Big[\sup_{r\in\mc{R}}\sum_{k=1}^K \sigma_k \ell_{\text{reg}}(m_{k,j}^f, b_k)1_{r(x_k)=j}\Big] & \leq  \sum_{j=1}^J \Big(\max \ell_{reg}(m_{j}^f, t)\mathfrak{R}_K({\mc{R}})\Big)
    \end{aligned}
\end{equation}
Then, it leads to:
\begin{equation*}
    \begin{aligned}
        \mathfrak{R}_K(\mc{L}_{\text{def}}) & \leq \frac{1}{2}\mathfrak{R}_K(\mc{H}) + \mathfrak{R}_K(\mc{F}) + \sum_{j=1}^J \Omega(m_j^h, y) + \Big(\sum_{j=1}^J\max\ell_{\text{reg}}(m_j^f, t) + 2\Big) \mathfrak{R}_K(\mc{R})\\
    \end{aligned}
\end{equation*}
with $\Omega(m_j^h, y)=\frac{\mc{D}(m_j^h \neq y)}{2} \exp\left(-\frac{K \mc{D}(m_j^h \neq y)}{8}\right) + \mathcal{R}_{K \mc{D}(m_j^h \neq y)/2}(\mathcal{R})$
\end{proof}

\section{Experiments} \label{appendix:exp}

\subsection{PascalVOC Experiment}\label{exp:pascal}
Since an image may contain multiple objects, our deferral rule is applied at the level of the entire image \( x \in \mathcal{X} \), ensuring that the approach remains consistent with real-world scenarios.

\begin{table}[H]\label{table:results_agent_pascal}
\centering
\begin{tabular}{@{}lccccc@{}}
\toprule
 & Model  & M$_1$ & M$_2$   \\ 
\midrule
mAP & $39.5$  & $43.3$ & $52.8$ \\
\bottomrule
\end{tabular}
\caption{Agent accuracies on the CIFAR-100 validation set. Since the training and validation sets are pre-determined in this dataset, the agents' knowledge remains fixed throughout the evaluation.}
\label{experts_agent_pascal}
\end{table}

\begin{table}[H]
\centering
\resizebox{0.65\textwidth}{!}{ 
\begin{tabular}{@{}ccccc@{}}
\toprule
Cost $\beta_2$ & mAP (\%)  & Model Allocation (\%)  & Expert 1 Allocation (\%)  & Expert 2 Allocation (\%)  \\
\midrule
0.01 & 52.8 $\pm$ 0.0 & 0.0 $\pm$ 0.0   & 0.0 $\pm$ 0.0   & 100.0 $\pm$ 0.0 \\
0.05 & 52.5 $\pm$ 0.1 & 7.3 $\pm$ 0.8   & 0.0 $\pm$ 0.0   & 92.7 $\pm$ 0.3 \\
0.1  & 49.1 $\pm$ 0.6 & 48.0 $\pm$ 0.7  & 0.0 $\pm$ 0.0   & 52.0 $\pm$ 0.2 \\
0.15 & 44.2 $\pm$ 0.4 & 68.1 $\pm$ 0.3  & 19.7 $\pm$ 0.4  & 12.2 $\pm$ 0.1 \\
0.2  & 42.0 $\pm$ 0.2 & 77.5 $\pm$ 0.2  & 22.5 $\pm$ 0.5  & 0.0 $\pm$ 0.0 \\
0.3  & 40.1 $\pm$ 0.2 & 98.1 $\pm$ 0.0  & 1.9 $\pm$ 0.1   & 0.0 $\pm$ 0.0 \\
0.5  & 39.5 $\pm$ 0.0 & 100.0 $\pm$ 0.0 & 0.0 $\pm$ 0.0   & 0.0 $\pm$ 0.0 \\
\bottomrule
\end{tabular}}
\caption{Detailed results across different cost values $\beta_2$. Errors represent the standard deviation over multiple runs.}
\label{table:results_allocation}
\end{table}

\subsection{MIMIC-IV Experiments} \label{appendix:ehr} \label{appendix_mimic}

MIMIC-IV \citep{johnson2023mimic} is a large collection of de-identified health-related data covering over forty thousand patients who stayed in critical care units. This dataset includes a wide variety of information, such as demographic details, vital signs, laboratory test results, medications, and procedures. For our analysis, we focus specifically on features related to \textit{procedures}, which correspond to medical procedures performed during hospital visits, and \textit{diagnoses} received by the patients.

Using these features, we address two predictive tasks: (1) a classification task to predict whether a patient will die during their next hospital visit based on clinical information from the current visit, and (2) a regression task to estimate the length of stay for the current hospital visit based on the same clinical information.

A key challenge in this task is the severe class imbalance, particularly in predicting mortality. To mitigate this issue, we sub-sample the negative mortality class, retaining a balanced dataset with $K = 5995$ samples, comprising 48.2\% positive mortality cases and 51.8\% negative mortality cases. Our model is trained on 80\% of this dataset, while the remaining 20\% is held out for validation. To ensure consistency in the results, we fixed the training and validation partitions.

\begin{table}[ht]
\centering
\begin{tabular}{@{}lcccc@{}}
\toprule
 & Model & M\(_1 \) & M\(_2 \) \\ 
\midrule
Accuracy  & \( 60.0 \) & \( 39.7 \) & $46.2$\\
Smooth L1 & \( 1.45 \) & \( 2.31 \) & $1.92$\\
\bottomrule
\end{tabular}
\caption{Performance of the agents on the MIMIC-IV dataset, evaluated in terms of accuracy and Smooth L1 loss. We fixed the training/validation set such that the agents' knowledge remains fixed throughout the evaluation.}
\label{experts_ehr}
\end{table}

\end{document}